\documentclass[11pt]{article}

\usepackage[preprint]{acl}

\usepackage{times}
\usepackage{latexsym}

\usepackage[T1]{fontenc}

\usepackage[utf8]{inputenc}

\usepackage{microtype}

\usepackage{inconsolata}

\usepackage{graphicx}

\usepackage[svgnames]{xcolor}
\usepackage{amsmath}
\usepackage{amssymb}
\usepackage{amsthm}
\usepackage{colortbl}
\usepackage{tcolorbox}
\usepackage{booktabs}
\usepackage{tikz}
\usetikzlibrary{positioning, backgrounds, fit, shadows}

\usepackage{multirow}
\usepackage{makecell}
\usepackage{enumitem}
\usepackage[subtle]{savetrees}

\newtheorem{definition}{Definition}

\newtheorem{theorem}{Theorem}

\graphicspath{{./figures/}}
\newcommand{\figCrossLingualCulturalInconsistency}{
    \begin{figure}[t]
        \centering
        \resizebox{\columnwidth}{!}{%
            \begin{tikzpicture}[
                node distance=0.4cm and 0.4cm,
                font=\small,
                sysbox/.style={
                    draw=gray!80, fill=gray!10, rounded corners=2pt,
                    align=center, inner sep=5pt, line width=0.8pt
                },
                userbox/.style={
                    draw=cyan!60!black, fill=cyan!5, rounded corners=2pt,
                    align=left, inner sep=5pt, text width=3.55cm, line width=0.8pt
                },
                modelnode/.style={
                    circle, draw=black!70, fill=white, minimum size=0.9cm,
                    line width=0.8pt, drop shadow, inner sep=2pt
                },
                outputbox/.style={
                    draw=green!60!black, fill=green!5, rounded corners=2pt,
                    align=center, minimum width=2.75cm, inner sep=5pt, line width=0.8pt
                },
                arrow/.style={-latex, thick, color=black!70}
            ]

                \node[sysbox] (sys) {\textbf{System Persona:} ``User is British''};

                \node[userbox, below left=0.4cm and -2.2cm of sys] (user_en) {\textbf{User:} ``Which writer is\\studied in literature class?''};
                \node[userbox, below right=0.4cm and -2.2cm of sys] (user_es) {\textbf{User:} ``¿Qué escritor se\\estudia en literatura?''};

                \node[modelnode, below=0.35cm of user_en] (ai_en) {\includegraphics[width=0.55cm]{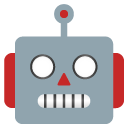}};
                \node[modelnode, below=0.35cm of user_es] (ai_es) {\includegraphics[width=0.55cm]{robot_emoji.png}};

                \node[outputbox, below=0.35cm of ai_en] (out_en) {\textbf{AI:} ``Shakespeare''};
                \node[outputbox, below=0.35cm of ai_es] (out_es) {\textbf{AI:} ``Cervantes''\vphantom{Shakespeare}};

                \draw[arrow] (sys.south) -- ++(0,-0.1333) -| (user_en.north);
                \draw[arrow] (sys.south) -- ++(0,-0.1333) -| (user_es.north);

                \draw[arrow] (user_en) -- (ai_en);
                \draw[arrow] (user_es) -- (ai_es);

                \draw[arrow] (ai_en) -- (out_en);
                \draw[arrow] (ai_es) -- (out_es);

                \draw[dashed, red, thick] (out_en) -- node[midway, fill=white, inner sep=2pt, label={[text=red, font=\footnotesize\bfseries, label distance=1pt]above:Inconsistency}] {\Large $\neq$} (out_es);
            \end{tikzpicture}%
        }
        \caption{Visualisation of Cross-lingual Cultural Inconsistency. Although the user persona is explicit, the model's output shifts based on the prompt's language, exposing implicit language-driven personalisation biases.}
        \label{fig:cross_lingual_cultural_inconsistency}
    \end{figure}
}

\newcommand{\figDatasetConstruction}{
    \begin{figure}[t]
        \centering
        \resizebox{1\columnwidth}{!}{%
            \begin{tikzpicture}[
                >=latex,
                font=\footnotesize,
                node distance=0.4cm and 0.5cm,
                basebox/.style={align=center, inner sep=3pt, text width=0.49\columnwidth, minimum height=0.95cm, rounded corners=3pt, drop shadow={opacity=0.1}},
                box/.style={basebox, draw=blue!70!black, fill=blue!5},
                stepbox/.style={basebox, draw=gray!70, fill=gray!5},
                endbox/.style={basebox, draw=green!60!black, fill=green!5},
                arrow/.style={->, thick, color=gray!80!black}
            ]

                \node[box] (start) {\textbf{BLEnD Benchmark}\\MCQ \& SAQ Subsets};

                \node[stepbox, right=of start] (step1) {\textbf{1. Question Extraction}\\Across 8 languages};

                \node[stepbox, below=of step1] (step2) {\textbf{2. Prompt Neutralisation}\\Remove country refs};

                \node[stepbox, left=of step2] (step3) {\textbf{3. Option Processing}\\Sample \& translate};

                \node[stepbox, below=of step3] (step4) {\textbf{4. Data Splitting}\\70/10/20 parallel split};

                \node[endbox, right=of step4] (end) {\textbf{Final Multilingual}\\\textbf{Parallel Dataset}};

                \draw[arrow] (start) -- (step1);
                \draw[arrow] (step1) -- (step2);
                \draw[arrow] (step2) -- (step3);
                \draw[arrow] (step3) -- (step4);
                \draw[arrow] (step4) -- (end);
            \end{tikzpicture}
        }
        \caption{Flowchart of the dataset construction process.}
        \label{fig:dataset_construction}
    \end{figure}
}

\newcommand{\figBLEnDMCQFormat}{
    \begin{figure}[htb]
        \centering

        \begin{tcolorbox}
            \textbf{Prompt:}

            \vspace{1em}

            What is a common snack for preschool kids in Assam? Without any explanation, choose only one from the given alphabet choices(e.g., A, B, C). Provide as JSON format: \{``answer\_choice'':``''\}

            \vspace{1em}

            A. finger-shaped biscuit\\
            B. jam sandwiches\\
            C. semolina\\
            D. toast

            \vspace{1em}

            Answer:

            \bigskip

            \textbf{Choice countries:}

            \vspace{1em}

            \{``A'': ``North\_Korea'', ``B'': ``Algeria'', ``C'': ``Assam'', ``D'': ``Greece''\}
        \end{tcolorbox}
        \caption{Format of the multiple-choice questions (MCQ) in the BLEnD benchmark. Each question consists of a prompt that includes a culturally-sensitive query and a set of answer choices, each associated with a specific country. In our work, we remove the country mention from the prompt to ensure that the question is culturally neutral.}
        \label{fig:blend_mcq_format}
    \end{figure}
}

\newcommand{\figCThreePO}{
    \begin{figure*}[t]
        \centering
        \resizebox{0.95\textwidth}{!}{%
            \begin{tikzpicture}[
                >=latex,
                font=\small,
                promptbox/.style={draw=cyan!70!black, fill=cyan!5, rounded corners=3pt, align=left, inner sep=5pt, text width=3.5cm, drop shadow={opacity=0.1}},
                ansbox/.style={draw=gray!70, fill=gray!5, rounded corners=3pt, align=center, inner sep=5pt, text width=1.8cm, drop shadow={opacity=0.1}},
                modelbox/.style={draw=purple!70!black, fill=purple!10, rounded corners=5pt, align=center, inner sep=10pt, font=\bfseries, drop shadow={opacity=0.2}},
                consensusbox/.style={draw=orange!80!black, fill=orange!10, rounded corners=3pt, align=center, inner sep=6pt, text width=1.8cm, drop shadow={opacity=0.1}},
                chosenbox/.style={draw=green!60!black, fill=green!5, rounded corners=2pt, align=left, inner sep=4pt, text width=3.5cm, drop shadow={opacity=0.1}},
                rejectedbox/.style={draw=red!60!black, fill=red!5, rounded corners=2pt, align=left, inner sep=4pt, text width=3.5cm, drop shadow={opacity=0.1}},
                pairbox/.style={rounded corners=4pt, inner sep=6pt},
                arrow/.style={->, thick, color=gray!80!black, rounded corners},
                labeltext/.style={font=\footnotesize\bfseries, color=gray!70!black}
            ]

                \node[modelbox, minimum height=3.0cm, text width=0.7cm] (base_llm) at (0,0) {\includegraphics[width=0.7cm]{robot_emoji.png}};

                \node[promptbox] (p_en) at (-3.35, 1.1) {\textbf{EN:} ``Which writer is studied in literature class?''};
                \node[font=\Large] (p_dots) at (-3.35, 0.15) {\vdots};
                \node[promptbox] (p_es) at (-3.35, -1.1) {\textbf{ES:} ``¿Qué escritor se estudia en literatura?''};

                \foreach \lang in {en, es} {
                    \draw[arrow] (p_\lang.east) -- (p_\lang.east -| base_llm.west);
                }

                \node[ansbox] (a_en) at (2.38, 1.1) {``Shakespeare''};
                \node[font=\Large] (a_dots) at (2.38, 0.15) {\vdots};
                \node[ansbox] (a_es) at (2.38, -1.1) {``Cervantes''};

                \foreach \lang in {en, es} {
                    \draw[arrow] (base_llm.east |- a_\lang.west) -- (a_\lang.west);
                }

                \node[consensusbox] (consensus) at (5.4, 0) {\textbf{Consensus:}\\``Shakespeare''};

                \coordinate (c_in) at ([xshift=-0.25cm]consensus.west);
                \foreach \lang in {en, es} {
                    \draw[thick, color=gray!80!black] (a_\lang.east) -- ++(0.15,0) |- (c_in);
                }
                \draw[arrow] (c_in) -- (consensus.west);

                \node[chosenbox] (c_en) at (9.2, 1.65) {\textbf{Chosen:}\\``Shakespeare'' (Consensus)};
                \node[rejectedbox] (r_en) at (9.2, 0.60) {\textbf{Rejected:}\\``Dante'' (Random)};
                \begin{scope}[on background layer]
                    \node[pairbox, fit=(c_en) (r_en)] (pair_en) {};
                \end{scope}
                \node[labeltext, align=left, anchor=south west, inner sep=2pt, scale=0.9, transform shape] (lbl_en) at (4.95, 1.12) {Language:\\EN (Agreed)};

                \node[chosenbox] (c_es) at (9.2, -0.60) {\textbf{Chosen:}\\``Shakespeare'' (Consensus)};
                \node[rejectedbox] (r_es) at (9.2, -1.65) {\textbf{Rejected:}\\``Cervantes'' (Divergent)};
                \begin{scope}[on background layer]
                    \node[pairbox, fit=(c_es) (r_es)] (pair_es) {};
                \end{scope}
                \node[labeltext, align=left, anchor=north west, inner sep=2pt, scale=0.9, transform shape] (lbl_es) at (4.95, -1.12) {Language:\\ES (Diverged)};

                \coordinate (out_point) at ([xshift=0.35cm]consensus.east);
                \draw[arrow, rounded corners=2pt] (consensus.east) -- (out_point) |- (pair_en.west -| c_en.west);
                \draw[arrow, rounded corners=2pt] (consensus.east) -- (out_point) |- (pair_es.west -| c_es.west);

                \node[modelbox, fill=green!15, draw=green!60!black, minimum height=3cm, text width=1.8cm] (final_llm) at (13.3, 0) {\includegraphics[width=0.7cm]{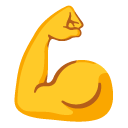}\includegraphics[width=0.7cm]{robot_emoji.png}\\Consistent};
                \node[above=0.1cm of final_llm, font=\bfseries, color=green!70!black] {DPO};

                \coordinate (dpo_point) at (11.4, 0);
                \draw[thick, color=gray!80!black] (pair_en.east -| c_en.east) -- ++(3.5pt,0) |- (dpo_point);
                \draw[thick, color=gray!80!black] (pair_es.east -| c_es.east) -- ++(3.5pt,0) |- (dpo_point);
                \draw[arrow] (dpo_point) -- (final_llm.west);

                \coordinate (bot_ref1) at (0, -2.15);
                \coordinate (top_ref1) at (0, 2.05);
                \coordinate (bot_ref2) at (7.55, -2.15);
                \coordinate (top_ref2) at (7.55, 2.05);
                \coordinate (bot_ref3) at (13.3, -2.15);
                \coordinate (top_ref3) at (13.3, 2.05);

                \begin{scope}[on background layer]
                    \node[fill=blue!3, draw=blue!20, thick, rounded corners=8pt, fit=(p_en) (p_es) (a_en) (a_es) (base_llm) (top_ref1) (bot_ref1), inner sep=7pt] (bg1) {};

                    \node[fill=yellow!5, draw=yellow!30, thick, rounded corners=8pt, fit=(consensus) (pair_en.west) (c_en.east) (lbl_en) (lbl_es) (top_ref2) (bot_ref2), inner sep=7pt] (bg2) {};

                    \node[fill=green!3, draw=green!20, thick, rounded corners=8pt, fit=(final_llm) (top_ref3) (bot_ref3), inner sep=7pt] (bg3) {};
                \end{scope}

                \node[anchor=north west, font=\footnotesize\bfseries, color=blue, inner sep=6pt] at (bg1.north west) {(1)};
                \node[anchor=north west, font=\footnotesize\bfseries, color=yellow!70!black, inner sep=6pt] at (bg2.north west) {(2)};
                \node[anchor=north west, font=\footnotesize\bfseries, color=green!70!black, inner sep=6pt] at (bg3.north west) {(3)};
            \end{tikzpicture}%
        }
        \caption{Overview of C-3PO. \textcolor{blue}{\textbf{(1)}} The base model generates answers for a culturally sensitive question across multiple languages. \textcolor{yellow!70!black}{\textbf{(2)}} A cross-lingual consensus is extracted to construct preference pairs: for consensual languages, a random incorrect option serves as the rejected response; for divergent languages, the model's actual output is rejected. \textcolor{green!70!black}{\textbf{(3)}} The model is fine-tuned using DPO with multilingual parallel batches.}
        \label{fig:c_3po}
    \end{figure*}
}

\newcommand{\figConsistencyVariationQwen}{
    \begin{figure}[t]
        \includegraphics[width=0.95\columnwidth]{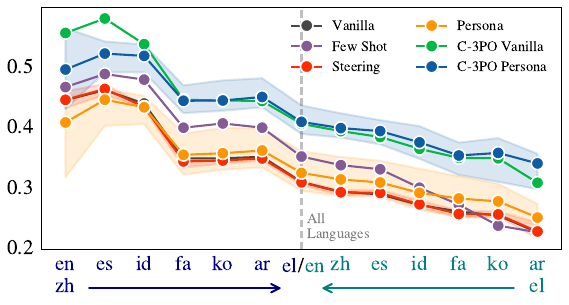}
        \caption{$\kappa_S$ variation (Qwen-2.5-3B) as languages are incrementally added by resource level: {\color{NavyBlue} higher$\rightarrow$lower}, {\color{teal} higher$\leftarrow$lower}. Bands indicate ranges across personas.}
        \label{fig:consistency_variation_qwen}
    \end{figure}
}

\newcommand{\figConsistencyVariationPlots}{
    \begin{figure*}[htb]
        \centering
        \includegraphics[width=0.9\textwidth]{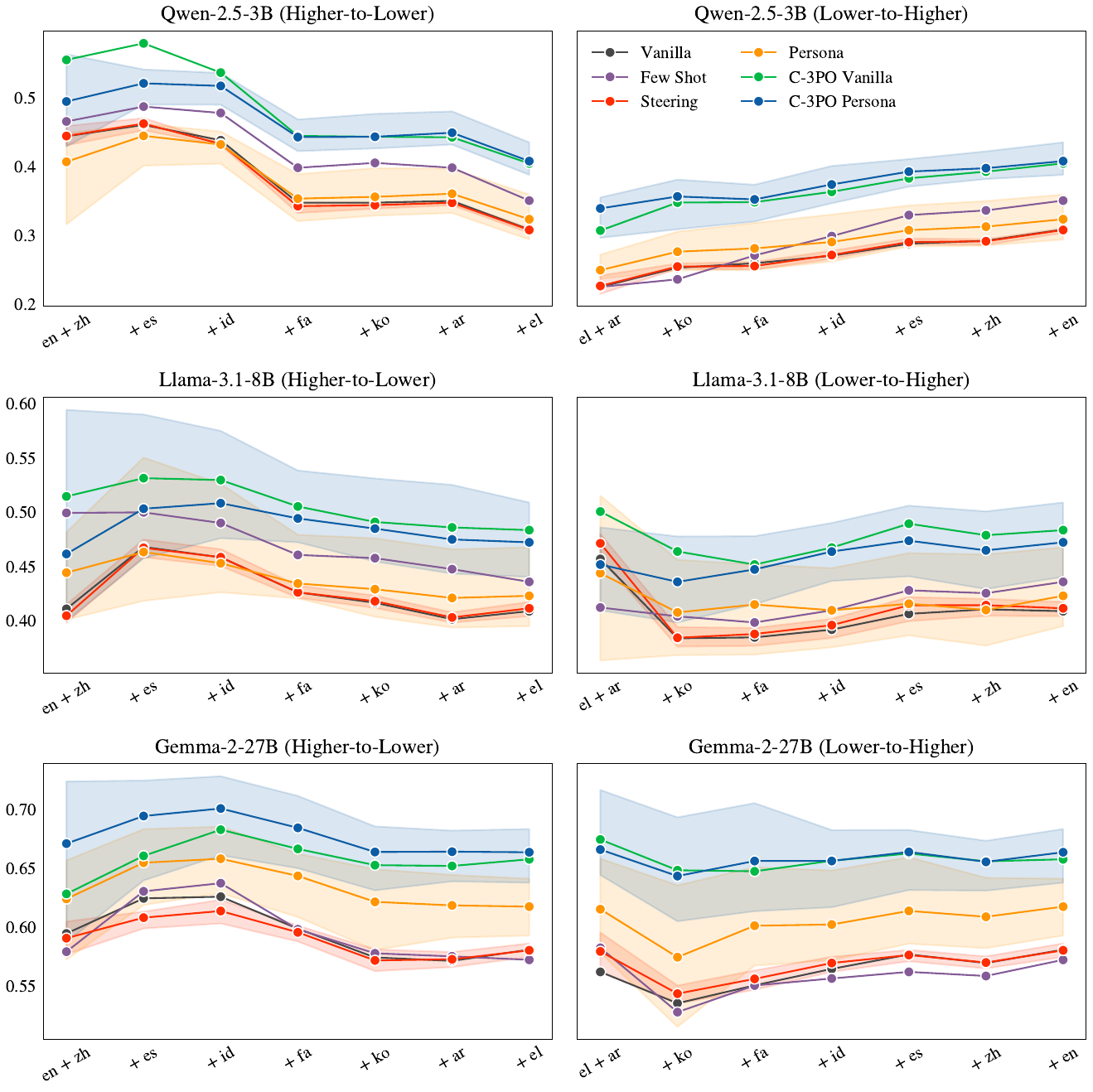}
        \caption{Singleton Fleiss's $\kappa_S$ consistency variation when adding languages in different orders for all models. Languages are ordered according to their resource level. Persona Vector Steering, Persona Prompting and C-3PO Persona show ranges across the eight personas.}
        \label{fig:consistency_variation_plots}
    \end{figure*}
}

\newcommand{\figLayerWiseKappa}{
    \begin{figure}[htb]
        \raggedleft
        \includegraphics[width=0.9\columnwidth]{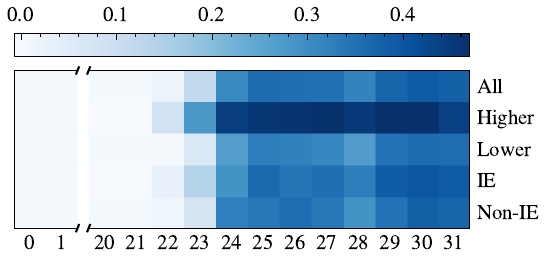}
        \caption{$\kappa_S$ consistency analysis across layers and language groups for Llama-3.1-8B. IE = Indo-European.}
        \label{fig:layer_wise_kappa}
    \end{figure}
}

\newcommand{\figLayerWiseCountryPredictions}{
    \begin{figure}[htb]
        \includegraphics[width=0.97\linewidth]{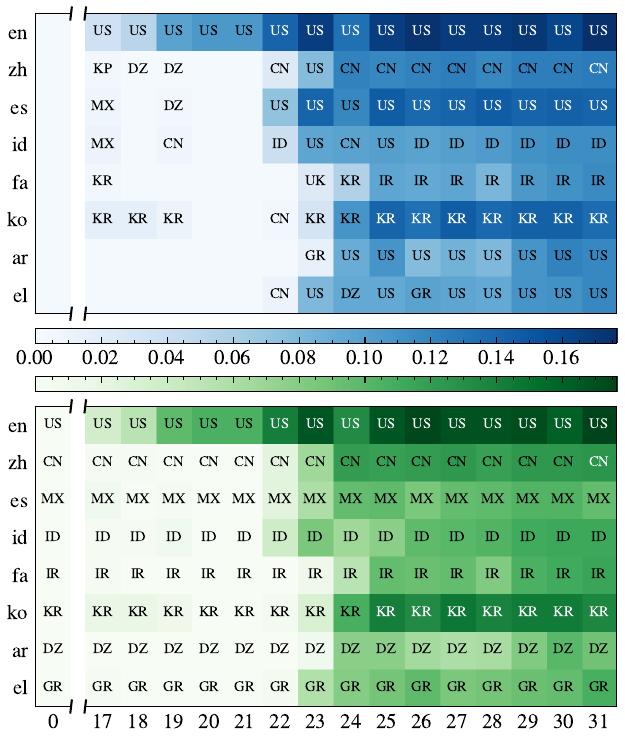}
        \caption{Cultural personalisation across layers for Llama-3.1-8B. Top: Frequency of the most predicted country. Bottom: Frequency of the language's stereotypical country (e.g., Indonesia for Indonesian).}
        \label{fig:layer_wise_country_predictions}
    \end{figure}
}

\newcommand{\figCombinedLayerSlopeAnalysis}{
    \begin{figure}[t]
        \includegraphics[width=\linewidth]{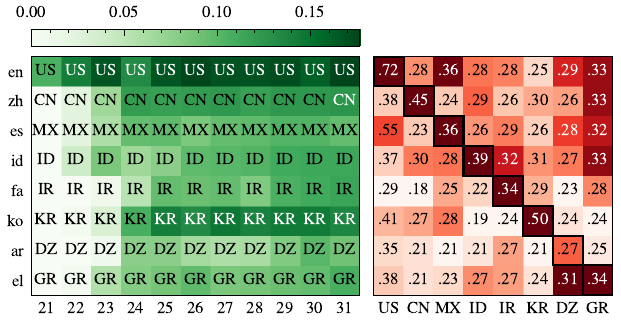}
        \caption{Cultural personalisation effect across layers for Llama-3.1-8B. {\color{ForestGreen} Left:} Frequency of the language's stereotypical country (e.g., Indonesia for Indonesian). {\color{purple} Right:} Slope (\%) of prediction frequency across layers per language-country pair, normalised column-wise, with boxes on each language's stereotypical country.}
        \label{fig:combined_layer_slope_analysis}
    \end{figure}
}

\newcommand{\figCrossModelSlopeAnalysis}{
    \begin{figure*}[t]
        \centering
        \includegraphics[width=0.95\linewidth]{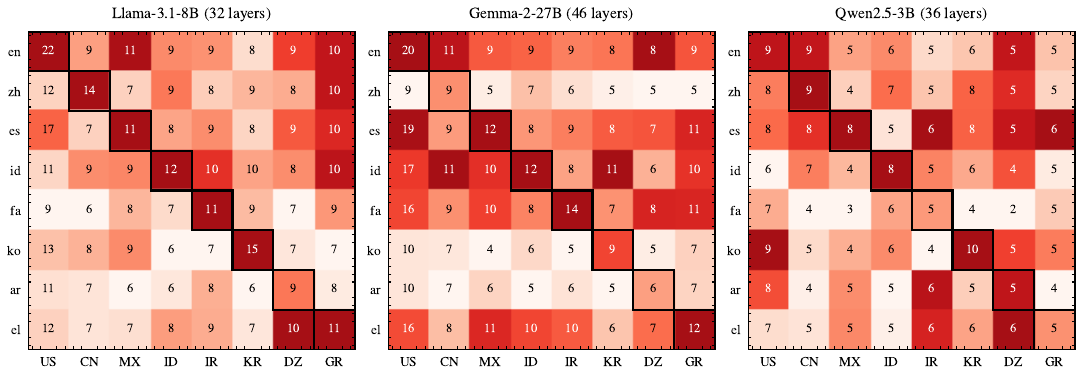}
        \caption{Cultural personalisation across models. Each cell is the change in the frequency with which a country is predicted, from the first to the last layer (percentage points), when prompting in a given language; per-layer slopes are rescaled by depth so that models with different numbers of layers are comparable. Colour is normalised per column, and boxes mark each language's stereotypical country.}
        \label{fig:cross_model_slope_analysis}
    \end{figure*}
}

\newcommand{\figBLEnDMCQFormatNoJSON}{
    \begin{figure}[htb]
        \centering

        \begin{tcolorbox}
            What is a common snack for preschool kids in Assam? Without any explanation, choose only one from the given alphabet choices(e.g., A, B, C).

            \vspace{1em}

            A. finger-shaped biscuit\\
            B. jam sandwiches\\
            C. semolina\\
            D. toast

            \vspace{1em}

            Answer:
        \end{tcolorbox}
        \caption{Format of the multiple-choice questions (MCQ) in the BLEnD benchmark without the JSON answer format. This format is used for the layer-wise early-decoding interpretability analysis in \autoref{sec:interpretability_analysis} to allow for direct extraction of the model's choice at each layer without needing to parse a JSON structure.}
        \label{fig:blend_mcq_format_no_json}
    \end{figure}
}

\newcommand{\figTreatQuestionPrompt}{
    \begin{figure}[htb]
        \centering
        \begin{tcolorbox}
            Remove the mention to the country or region in the following question. Keep the question as is, but without any mention to the country or region, ensuring that the question is grammatically correct. For example, "What is a common snack for preschool kids in Algeria?" should be modified to "What is a common snack for preschool kids?". For example, "Which city in Mexico hosts festivals the most?" should be modified to "Which city hosts festivals the most?". Do not answer the question, just modify it. The input question is in \{language\}. Return the modified question in \{language\}. Return only the modified question without any additional explanation.
        \end{tcolorbox}
        \caption{System prompt used to remove country mentions from questions to ensure cultural neutrality.}
        \label{fig:treat_question_prompt}
    \end{figure}
}

\newcommand{\figTranslateOptionsPrompt}{
    \begin{figure}[htb]
        \centering
        \begin{tcolorbox}
            Translate the following answer options from English to \{language\}. As additional context, the question is: "\{question\}". Do not choose an answer, just translate the options. Keep the answer option labels in English (A, B, C...) but the answer options text should be *completely* in \{language\}. Do not add English words, just \{language\} words. Return only the translated options without any additional explanation.
        \end{tcolorbox}
        \caption{System prompt used to translate answer options to the target language.}
        \label{fig:translate_options_prompt}
    \end{figure}
}

\newcommand{\figPersonaPrompt}{
    \begin{figure}[htb]
        \centering
        \begin{tcolorbox}
            You are a person from \{country\}. Answer the question based on the cultural norms and common practices of \{country\}.
        \end{tcolorbox}
        \caption{System prompt template used for the Persona Prompting baseline.}
        \label{fig:persona_prompt}
    \end{figure}
}

\newcommand{\tabMainResults}{
    \begin{table*}[htb]
        \centering
        \resizebox{\linewidth}{!}{
            \begin{tabular}{ccccccccccccccc}
                \toprule
                \multirow{2}{*}{\textbf{Model}} & \multirow{2}{*}{\makecell{\textbf{Language}\\\textbf{Group}}} & \multirow{2}{*}{\textbf{Vanilla}} & \multirow{2}{*}{\makecell{\textbf{Few}\\\textbf{Shot}}} & \multicolumn{3}{c}{\textbf{Steering}} & \multicolumn{3}{c}{\textbf{Persona}} & \multirow{2}{*}{\makecell{\textbf{C-3PO} \\ \textbf{Vanilla}}} & \multicolumn{3}{c}{\textbf{C-3PO Persona}} & \multirow{2}{*}{\makecell{\textbf{Max} \\ \textbf{$\sigma$}}} \\
                \cmidrule(lr){5-7} \cmidrule(lr){8-10} \cmidrule(lr){12-14}
                & & & & \textbf{Min} & \textbf{Avg} & \textbf{Max} & \textbf{Min} & \textbf{Avg} & \textbf{Max} & & \textbf{Min} & \textbf{Avg} & \textbf{Max} & \\
                \midrule
                \multirow[c]{5}{*}{\makecell{Qwen \\ 2.5-3B}} & All & {\cellcolor[HTML]{E3EEF8}} \color{black} .309 & {\cellcolor[HTML]{94C4DF}} \color{black} .351 & {\cellcolor[HTML]{EAF2FB}} \color{black} .304 \,{\footnotesize CN} & {\cellcolor[HTML]{E3EEF9}} \color{black} .309 & {\cellcolor[HTML]{E0ECF8}} \color{black} .311 \,{\footnotesize MX} & {\cellcolor[HTML]{F7FBFF}} \color{black} .295 \,{\footnotesize ID} & {\cellcolor[HTML]{CEE0F2}} \color{black} .324 & {\cellcolor[HTML]{7AB6D9}} \color{black} .360 \,{\footnotesize US} & {\cellcolor[HTML]{1A68AE}} \color{white} .405 & {\cellcolor[HTML]{3787C0}} \color{white} .389 \,{\footnotesize CN} & {\cellcolor[HTML]{1562A9}} \color{white} \bfseries .409 & {\cellcolor[HTML]{08306B}} \color{white} .436 \,{\footnotesize US} & .028 \\
                & High-Res & {\cellcolor[HTML]{ABD0E6}} \color{black} .461 & {\cellcolor[HTML]{72B2D8}} \color{black} .488 & {\cellcolor[HTML]{B9D6EA}} \color{black} .454 \,{\footnotesize KR} & {\cellcolor[HTML]{A9CFE5}} \color{black} .463 & {\cellcolor[HTML]{99C7E0}} \color{black} .471 \,{\footnotesize GR} & {\cellcolor[HTML]{F7FBFF}} \color{black} .402 \,{\footnotesize ID} & {\cellcolor[HTML]{C7DCEF}} \color{black} .446 & {\cellcolor[HTML]{A9CFE5}} \color{black} .463 \,{\footnotesize CN} & {\cellcolor[HTML]{08306B}} \color{white} \bfseries .580 & {\cellcolor[HTML]{6CAED6}} \color{black} .491 \,{\footnotesize ID} & {\cellcolor[HTML]{3585BF}} \color{white} .522 & {\cellcolor[HTML]{1967AD}} \color{white} .542 \,{\footnotesize US} & .038 \\
                & Low-Res & {\cellcolor[HTML]{EBF3FB}} \color{black} .271 & {\cellcolor[HTML]{C1D9ED}} \color{black} .300 & {\cellcolor[HTML]{F2F8FD}} \color{black} .266 \,{\footnotesize GR} & {\cellcolor[HTML]{EAF2FB}} \color{black} .272 & {\cellcolor[HTML]{E1EDF8}} \color{black} .278 \,{\footnotesize MX} & {\cellcolor[HTML]{F7FBFF}} \color{black} .263 \,{\footnotesize CN} & {\cellcolor[HTML]{CFE1F2}} \color{black} .291 & {\cellcolor[HTML]{6FB0D7}} \color{black} .331 \,{\footnotesize US} & {\cellcolor[HTML]{2676B8}} \color{white} .364 & {\cellcolor[HTML]{4695C8}} \color{black} .348 \,{\footnotesize CN} & {\cellcolor[HTML]{1562A9}} \color{white} \bfseries .375 & {\cellcolor[HTML]{08306B}} \color{white} .402 \,{\footnotesize US} & .031 \\
                & Indo-Eu & {\cellcolor[HTML]{EFF6FC}} \color{black} .244 & {\cellcolor[HTML]{A5CDE3}} \color{black} .280 & {\cellcolor[HTML]{F7FBFF}} \color{black} .239 \,{\footnotesize US} & {\cellcolor[HTML]{E7F1FA}} \color{black} .248 & {\cellcolor[HTML]{DDEAF7}} \color{black} .254 \,{\footnotesize GR} & {\cellcolor[HTML]{F6FAFF}} \color{black} .240 \,{\footnotesize ID} & {\cellcolor[HTML]{C6DBEF}} \color{black} .268 & {\cellcolor[HTML]{5BA3D0}} \color{black} .303 \,{\footnotesize US} & {\cellcolor[HTML]{084285}} \color{white} \bfseries .347 & {\cellcolor[HTML]{3787C0}} \color{white} .316 \,{\footnotesize CN} & {\cellcolor[HTML]{084F99}} \color{white} .341 & {\cellcolor[HTML]{08306B}} \color{white} .355 \,{\footnotesize DZ} & .033 \\
                & Non-Indo & {\cellcolor[HTML]{EEF5FC}} \color{black} .355 & {\cellcolor[HTML]{C1D9ED}} \color{black} .392 & {\cellcolor[HTML]{F7FBFF}} \color{black} .346 \,{\footnotesize KR} & {\cellcolor[HTML]{EFF6FC}} \color{black} .353 & {\cellcolor[HTML]{E8F1FA}} \color{black} .360 \,{\footnotesize IR} & {\cellcolor[HTML]{F1F7FD}} \color{black} .352 \,{\footnotesize CN} & {\cellcolor[HTML]{D3E3F3}} \color{black} .378 & {\cellcolor[HTML]{ABD0E6}} \color{black} .404 \,{\footnotesize US} & {\cellcolor[HTML]{4090C5}} \color{black} .455 & {\cellcolor[HTML]{4B98CA}} \color{black} .449 \,{\footnotesize ID} & {\cellcolor[HTML]{2676B8}} \color{white} \bfseries .472 & {\cellcolor[HTML]{08306B}} \color{white} .518 \,{\footnotesize US} & .034 \\
                \midrule
                \multirow[c]{5}{*}{\makecell{Llama \\ 3.1-8B}} & All & {\cellcolor[HTML]{DFECF7}} \color{black} .409 & {\cellcolor[HTML]{A4CCE3}} \color{black} .436 & {\cellcolor[HTML]{E7F1FA}} \color{black} .405 \,{\footnotesize CN} & {\cellcolor[HTML]{DBE9F6}} \color{black} .412 & {\cellcolor[HTML]{D0E2F2}} \color{black} .418 \,{\footnotesize US} & {\cellcolor[HTML]{F7FBFF}} \color{black} .396 \,{\footnotesize ID} & {\cellcolor[HTML]{C7DCEF}} \color{black} .424 & {\cellcolor[HTML]{3E8EC4}} \color{black} .468 \,{\footnotesize US} & {\cellcolor[HTML]{1C6AB0}} \color{white} \bfseries .484 & {\cellcolor[HTML]{95C5DF}} \color{black} .441 \,{\footnotesize ID} & {\cellcolor[HTML]{3484BF}} \color{white} .473 & {\cellcolor[HTML]{08306B}} \color{white} .510 \,{\footnotesize US} & .031 \\
                & High-Res & {\cellcolor[HTML]{B9D6EA}} \color{black} .469 & {\cellcolor[HTML]{75B4D8}} \color{black} .501 & {\cellcolor[HTML]{CADDF0}} \color{black} .459 \,{\footnotesize IR} & {\cellcolor[HTML]{BCD7EB}} \color{black} .468 & {\cellcolor[HTML]{ADD0E6}} \color{black} .475 \,{\footnotesize CN} & {\cellcolor[HTML]{F7FBFF}} \color{black} .419 \,{\footnotesize ID} & {\cellcolor[HTML]{C2D9EE}} \color{black} .464 & {\cellcolor[HTML]{1D6CB1}} \color{white} .551 \,{\footnotesize US} & {\cellcolor[HTML]{3989C1}} \color{white} \bfseries .532 & {\cellcolor[HTML]{CADEF0}} \color{black} .458 \,{\footnotesize ID} & {\cellcolor[HTML]{6DAFD7}} \color{black} .504 & {\cellcolor[HTML]{08306B}} \color{white} .591 \,{\footnotesize US} & .039 \\
                & Low-Res & {\cellcolor[HTML]{DBE9F6}} \color{black} .392 & {\cellcolor[HTML]{B7D4EA}} \color{black} .410 & {\cellcolor[HTML]{E8F1FA}} \color{black} .385 \,{\footnotesize CN} & {\cellcolor[HTML]{D4E4F4}} \color{black} .397 & {\cellcolor[HTML]{CADEF0}} \color{black} .402 \,{\footnotesize GR} & {\cellcolor[HTML]{F7FBFF}} \color{black} .376 \,{\footnotesize ID} & {\cellcolor[HTML]{B7D4EA}} \color{black} .410 & {\cellcolor[HTML]{3F8FC5}} \color{black} .449 \,{\footnotesize US} & {\cellcolor[HTML]{1663AA}} \color{white} \bfseries .468 & {\cellcolor[HTML]{60A7D2}} \color{black} .437 \,{\footnotesize ID} & {\cellcolor[HTML]{1C6BB0}} \color{white} .464 & {\cellcolor[HTML]{08306B}} \color{white} .491 \,{\footnotesize IR} & .032 \\
                & Indo-Eu & {\cellcolor[HTML]{DAE8F6}} \color{black} .426 & {\cellcolor[HTML]{D0E2F2}} \color{black} .432 & {\cellcolor[HTML]{E6F0F9}} \color{black} .419 \,{\footnotesize CN} & {\cellcolor[HTML]{D2E3F3}} \color{black} .431 & {\cellcolor[HTML]{BED8EC}} \color{black} .441 \,{\footnotesize US} & {\cellcolor[HTML]{F7FBFF}} \color{black} .408 \,{\footnotesize CN} & {\cellcolor[HTML]{B0D2E7}} \color{black} .447 & {\cellcolor[HTML]{2777B8}} \color{white} .496 \,{\footnotesize GR} & {\cellcolor[HTML]{2A7AB9}} \color{white} \bfseries .494 & {\cellcolor[HTML]{C1D9ED}} \color{black} .440 \,{\footnotesize CN} & {\cellcolor[HTML]{3484BF}} \color{white} .490 & {\cellcolor[HTML]{08306B}} \color{white} .529 \,{\footnotesize IR} & .035 \\
                & Non-Indo & {\cellcolor[HTML]{CADEF0}} \color{black} .387 & {\cellcolor[HTML]{77B5D9}} \color{black} .426 & {\cellcolor[HTML]{D7E6F5}} \color{black} .377 \,{\footnotesize DZ} & {\cellcolor[HTML]{D0E1F2}} \color{black} .383 & {\cellcolor[HTML]{CCDFF1}} \color{black} .386 \,{\footnotesize GR} & {\cellcolor[HTML]{F7FBFF}} \color{black} .351 \,{\footnotesize ID} & {\cellcolor[HTML]{CDDFF1}} \color{black} .385 & {\cellcolor[HTML]{519CCC}} \color{black} .443 \,{\footnotesize US} & {\cellcolor[HTML]{3585BF}} \color{white} \bfseries .458 & {\cellcolor[HTML]{B0D2E7}} \color{black} .402 \,{\footnotesize KR} & {\cellcolor[HTML]{4E9ACB}} \color{black} .444 & {\cellcolor[HTML]{08306B}} \color{white} .510 \,{\footnotesize US} & .035 \\
                \midrule
                \multirow[c]{5}{*}{\makecell{Gemma \\ 2-27B}} & All & {\cellcolor[HTML]{E8F1FA}} \color{black} .581 & {\cellcolor[HTML]{F7FBFF}} \color{black} .573 & {\cellcolor[HTML]{F3F8FE}} \color{black} .575 \,{\footnotesize DZ} & {\cellcolor[HTML]{E9F2FA}} \color{black} .581 & {\cellcolor[HTML]{DEEBF7}} \color{black} .587 \,{\footnotesize US} & {\cellcolor[HTML]{D3E3F3}} \color{black} .593 \,{\footnotesize MX} & {\cellcolor[HTML]{91C3DE}} \color{black} .618 & {\cellcolor[HTML]{4493C7}} \color{black} .642 \,{\footnotesize KR} & {\cellcolor[HTML]{1D6CB1}} \color{white} .658 & {\cellcolor[HTML]{4E9ACB}} \color{black} .638 \,{\footnotesize MX} & {\cellcolor[HTML]{125EA6}} \color{white} \bfseries .664 & {\cellcolor[HTML]{08306B}} \color{white} .684 \,{\footnotesize IR} & .031 \\
                & High-Res & {\cellcolor[HTML]{CFE1F2}} \color{black} .625 & {\cellcolor[HTML]{C6DBEF}} \color{black} .631 & {\cellcolor[HTML]{F7FBFF}} \color{black} .599 \,{\footnotesize DZ} & {\cellcolor[HTML]{E9F2FA}} \color{black} .608 & {\cellcolor[HTML]{E0ECF8}} \color{black} .614 \,{\footnotesize IR} & {\cellcolor[HTML]{D8E7F5}} \color{black} .619 \,{\footnotesize MX} & {\cellcolor[HTML]{82BBDB}} \color{black} .655 & {\cellcolor[HTML]{3686C0}} \color{white} .684 \,{\footnotesize IR} & {\cellcolor[HTML]{6FB0D7}} \color{black} .661 & {\cellcolor[HTML]{AED1E7}} \color{black} .640 \,{\footnotesize MX} & {\cellcolor[HTML]{1F6EB3}} \color{white} \bfseries .695 & {\cellcolor[HTML]{08306B}} \color{white} .725 \,{\footnotesize US} & .036 \\
                & Low-Res & {\cellcolor[HTML]{EAF3FB}} \color{black} .565 & {\cellcolor[HTML]{F7FBFF}} \color{black} .557 & {\cellcolor[HTML]{EEF5FC}} \color{black} .562 \,{\footnotesize IR} & {\cellcolor[HTML]{E3EEF8}} \color{black} .570 & {\cellcolor[HTML]{DAE8F6}} \color{black} .575 \,{\footnotesize CN} & {\cellcolor[HTML]{DFEBF7}} \color{black} .572 \,{\footnotesize ID} & {\cellcolor[HTML]{A1CBE2}} \color{black} .603 & {\cellcolor[HTML]{2676B8}} \color{white} .648 \,{\footnotesize KR} & {\cellcolor[HTML]{1966AD}} \color{white} \bfseries .657 & {\cellcolor[HTML]{74B3D8}} \color{black} .617 \,{\footnotesize ID} & {\cellcolor[HTML]{1966AD}} \color{white} .657 & {\cellcolor[HTML]{08306B}} \color{white} .683 \,{\footnotesize KR} & .034 \\
                & Indo-Eu & {\cellcolor[HTML]{C2D9EE}} \color{black} .627 & {\cellcolor[HTML]{F7FBFF}} \color{black} .596 & {\cellcolor[HTML]{CFE1F2}} \color{black} .620 \,{\footnotesize GR} & {\cellcolor[HTML]{C7DCEF}} \color{black} .624 & {\cellcolor[HTML]{BFD8ED}} \color{black} .627 \,{\footnotesize ID} & {\cellcolor[HTML]{C7DBEF}} \color{black} .625 \,{\footnotesize DZ} & {\cellcolor[HTML]{74B3D8}} \color{black} .652 & {\cellcolor[HTML]{2C7CBA}} \color{white} .679 \,{\footnotesize US} & {\cellcolor[HTML]{3787C0}} \color{white} .674 & {\cellcolor[HTML]{4896C8}} \color{black} .667 \,{\footnotesize ID} & {\cellcolor[HTML]{115CA5}} \color{white} \bfseries .693 & {\cellcolor[HTML]{08306B}} \color{white} .713 \,{\footnotesize US} & .033 \\
                & Non-Indo & {\cellcolor[HTML]{F7FBFF}} \color{black} .540 & {\cellcolor[HTML]{E5EFF9}} \color{black} .551 & {\cellcolor[HTML]{E8F1FA}} \color{black} .549 \,{\footnotesize ID} & {\cellcolor[HTML]{DEEBF7}} \color{black} .555 & {\cellcolor[HTML]{D0E2F2}} \color{black} .563 \,{\footnotesize US} & {\cellcolor[HTML]{D9E8F5}} \color{black} .558 \,{\footnotesize ID} & {\cellcolor[HTML]{A3CCE3}} \color{black} .584 & {\cellcolor[HTML]{3282BE}} \color{white} .622 \,{\footnotesize KR} & {\cellcolor[HTML]{0B559F}} \color{white} \bfseries .643 & {\cellcolor[HTML]{4191C6}} \color{black} .615 \,{\footnotesize MX} & {\cellcolor[HTML]{125DA6}} \color{white} .639 & {\cellcolor[HTML]{08306B}} \color{white} .660 \,{\footnotesize GR} & .034 \\
                \bottomrule
            \end{tabular}}
        \caption{Singleton Fleiss's $\kappa_S$ scores ($\uparrow$) across models, language groups and mitigation strategies. Steering, Persona Prompting and C-3PO Persona span eight personas, reporting minimum, average and maximum scores (indicated by country code). Max $\sigma$ is the maximum bootstrap standard deviation across methods. Colour scales are applied row-wise. Bold marks the best result per row, with each persona-based method represented by its average.}
        \label{tab:main_results}
    \end{table*}
}

\newcommand{\tabRelatedWork}{
    \begin{table}[t]
        \centering
        \footnotesize
        \setlength{\tabcolsep}{2pt}
        \resizebox{\columnwidth}{!}{%
            \begin{tabular}{lccccc}
                \toprule
                \textbf{Method} & \makecell[b]{\textbf{Cross-}\\\textbf{lingual}} & \makecell[b]{\textbf{No}\\\textbf{ground}\\\textbf{truth}} & \makecell[b]{\textbf{No}\\\textbf{pivot}\\\textbf{language}} & \makecell[b]{\textbf{No}\\\textbf{external}\\\textbf{model}} & \makecell[b]{\textbf{Explicit}\\\textbf{persona}} \\
                \midrule
                LMSI & $\times$ & $\times$ & -- & \checkmark & $\times$ \\
                Semantic Voting & $\times$ & \checkmark & -- & $\times$ & $\times$ \\
                CALM & \checkmark & $\times$ & \checkmark & \checkmark & $\times$ \\
                MAPO & \checkmark & $\times$ & $\times$ & $\times$ & $\times$ \\
                CM-Align & \checkmark & $\times$ & $\times$ & \checkmark & $\times$ \\
                \midrule
                \textbf{C-3PO (ours)} & \checkmark & \checkmark & \checkmark & \checkmark & \checkmark \\
                \bottomrule
            \end{tabular}}
        \caption{Consensus-based training methods. \textit{No ground truth}: the target task has no verifiable answer. \textit{No pivot language}: agreement is not measured against English or a dominant language (--: not applicable). \textit{No external model}: no translation or embedding model is needed to form the target.}
        \label{tab:related_work}
    \end{table}
}

\newcommand{\tabOpenEnded}{
    \begin{table}[t]
        \centering
        \small
        \setlength{\tabcolsep}{3.5pt}
        \begin{tabular}{llcccc}
            \toprule
            & & \multicolumn{2}{c}{\textbf{Vanilla prompt}} & \multicolumn{2}{c}{\textbf{Persona prompt}} \\
            \cmidrule(lr){3-4} \cmidrule(lr){5-6}
            \textbf{Model} & \textbf{Group} & Vanilla & C-3PO & Persona & C-3PO \\
            \midrule
            \multirow{5}{*}{\makecell[l]{Qwen\\2.5-3B}} & All & .069 & \textbf{.079} & .073 & \textbf{.085}$^{*}$ \\
            & Higher-Res & .104 & \textbf{.120} & .125 & \textbf{.140} \\
            & Lower-Res & .047 & \textbf{.051} & .049 & \textbf{.059}$^{*}$ \\
            & Indo-Eu & .046 & \textbf{.054} & .053 & \textbf{.063}$^{*}$ \\
            & Non-Indo & .081 & \textbf{.094} & .085 & \textbf{.097} \\
            \midrule
            \multirow{5}{*}{\makecell[l]{Llama\\3.1-8B}} & All & .111 & \textbf{.140} & .132 & \textbf{.143} \\
            & Higher-Res & .158 & \textbf{.230}$^{*}$ & .202 & \textbf{.203} \\
            & Lower-Res & .092 & \textbf{.121} & .107 & \textbf{.115} \\
            & Indo-Eu & .108 & \textbf{.151} & .151 & \textbf{.163} \\
            & Non-Indo & .117 & \textbf{.134} & .118 & \textbf{.120} \\
            \midrule
            \multirow{5}{*}{\makecell[l]{Gemma\\2-27B}} & All & .237 & \textbf{.283}$^{*}$ & .327 & \textbf{.366}$^{*}$ \\
            & Higher-Res & .260 & \textbf{.298} & .312 & \textbf{.367}$^{*}$ \\
            & Lower-Res & .222 & \textbf{.280}$^{*}$ & .337 & \textbf{.372}$^{*}$ \\
            & Indo-Eu & \textbf{.309} & .303 & .357 & \textbf{.406}$^{*}$ \\
            & Non-Indo & .195 & \textbf{.277}$^{*}$ & .305 & \textbf{.344}$^{*}$ \\
            \bottomrule
        \end{tabular}
        \caption{Open-ended $\kappa_S$ ($\uparrow$, scoring level L4). Persona columns average the eight personas. The higher value of each pair is in bold; $^{*}$: two-sided $p < .05$, paired bootstrap over the 79 base questions.}
        \label{tab:open_ended}
    \end{table}
}

\newcommand{\tabConsensusAblation}{
    \begin{table}[t]
        \centering
        \small
        \setlength{\tabcolsep}{3pt}
        \resizebox{\columnwidth}{!}{%
            \begin{tabular}{lcccccc}
                \toprule
                & \multicolumn{3}{c}{\textbf{Multiple choice}} & \multicolumn{3}{c}{\textbf{Open-ended}} \\
                \cmidrule(lr){2-4} \cmidrule(lr){5-7}
                \textbf{Consensus} & Qwen & Llama & Gemma & Qwen & Llama & Gemma \\
                \midrule
                C-3PO (plurality) & .387 & .505 & .669 & .076 & .135 & .301 \\
                \midrule
                Strict majority & +.004 & +.000 & $-$.081$^{*}$ & $-$.003 & $-$.004 & $-$.011 \\
                Soft vote & +.002 & $-$.006 & $-$.039$^{*}$ & +.006 & +.010 & +.012 \\
                \quad + confidence & $-$.002 & $-$.006 & $-$.022 & $-$.008 & +.002 & $-$.021 \\
                \quad + hard negatives & +.002 & $-$.006 & $-$.009 & $-$.006 & +.021 & +.000 \\
                \quad + margin gate & $-$.001 & +.012 & $-$.022 & $-$.003 & +.008 & +.002 \\
                \bottomrule
            \end{tabular}}
        \caption{Consensus-rule ablation. First row: $\kappa_S$ of C-3PO retrained with its plurality rule; other rows: change in $\kappa_S$ relative to it. Indented rows add each component to the soft vote in turn. $^{*}$: 95\% interval excludes zero, bootstrap over base questions.}
        \label{tab:consensus_ablation}
    \end{table}
}

\newcommand{\tabLlamaAnalysis}{
    \begin{table}[t]
        \centering
        \small
        \setlength{\tabcolsep}{4pt}
        \begin{tabular}{lccc}
            \toprule
            & \textbf{Persona} & \multicolumn{2}{c}{\textbf{BLEnD accuracy}} \\
            \cmidrule(lr){3-4}
            \textbf{Llama-3.1-8B} & \textbf{adherence} & Seen & Unseen \\
            \midrule
            Vanilla & 49.9 & 69.8 & 62.2 \\
            C-3PO & 43.1 & 62.0 & 56.7 \\
            C-3PO, soft consensus & 48.2 & 67.9 & 60.3 \\
            \bottomrule
        \end{tabular}
        \caption{Llama-3.1-8B's persona adherence and cultural knowledge (\%, $\uparrow$). The soft-consensus variant recovers most of the cost of C-3PO. Persona adherence is measured with persona prompting on each model.}
        \label{tab:llama_analysis}
    \end{table}
}

\newcommand{\tabCulturalBiasExpanded}{
    \begin{table*}[t]
        \resizebox{\linewidth}{!}{
            {
                \begin{tabular}{llccccccccccccccccc|cc}
                    \toprule
                    \textbf{Model} & \textbf{Method} & \textbf{AS} & \textbf{AZ} & \textbf{CN} & \textbf{DZ} & \textbf{ES} & \textbf{ET} & \textbf{GB} & \textbf{GR} & \textbf{ID} & \textbf{IR} & \textbf{KP} & \textbf{KR} & \textbf{MX} & \textbf{NG} & \textbf{US} & \textbf{WJ} & \textbf{XX} & \textbf{Mean} & \textbf{Std} \\
                    \midrule
                    \multirow{4}{*}{\makecell[l]{Qwen\\2.5-3B}} & Vanilla & 4.28 & 3.61 & 8.98 & 6.97 & 5.40 & 3.68 & 6.84 & 7.07 & 7.67 & 7.67 & 4.41 & 9.49 & 5.85 & 3.23 & 10.20 & 2.62 & 1.76 & 5.87 & 2.44 \\
                    & Few Shot & 4.73 & 3.68 & 7.90 & 7.70 & 5.72 & 4.09 & 7.13 & 7.90 & 7.80 & 7.42 & 3.93 & 7.77 & 5.15 & 3.04 & 11.73 & 2.56 & 1.73 & 5.88 & 2.51 \\
                    & C-3PO & 4.57 & 3.48 & 9.78 & 6.65 & 6.04 & 2.97 & 7.16 & 7.23 & 7.07 & 7.48 & 4.06 & 8.76 & 5.98 & 3.36 & 11.25 & 2.43 & 1.31 & 5.86 & 2.64 \\
                    & Preference & 4.17 & 2.94 & 11.84 & 6.53 & 3.31 & 2.84 & 6.91 & 8.33 & 7.86 & 6.63 & 3.22 & 9.09 & 6.82 & 1.89 & 14.30 & 2.37 & 0.95 & 5.88 & 3.57 \\
                    \midrule
                    \multirow{4}{*}{\makecell[l]{Llama\\3.1-8B}} & Vanilla & 3.87 & 3.55 & 7.77 & 7.29 & 6.43 & 2.43 & 6.23 & 9.37 & 7.07 & 8.47 & 3.13 & 8.18 & 7.80 & 2.59 & 11.41 & 2.05 & 2.01 & 5.86 & 2.82 \\
                    & Few Shot & 4.35 & 3.16 & 7.99 & 7.80 & 6.55 & 2.85 & 7.13 & 8.63 & 6.23 & 8.06 & 4.44 & 8.34 & 6.36 & 2.56 & 10.84 & 2.62 & 2.01 & 5.88 & 2.56 \\
                    & C-3PO & 3.64 & 3.26 & 7.23 & 7.00 & 6.33 & 3.01 & 6.30 & 8.89 & 6.23 & 7.83 & 4.03 & 7.64 & 8.92 & 2.27 & 12.02 & 1.47 & 2.01 & 5.77 & 2.84 \\
                    & Preference & 3.74 & 2.63 & 10.01 & 6.47 & 4.25 & 3.74 & 7.79 & 7.89 & 7.18 & 6.37 & 2.12 & 9.71 & 8.80 & 1.62 & 13.65 & 2.93 & 1.11 & 5.88 & 3.40 \\
                    \midrule
                    \multirow{4}{*}{\makecell[l]{Gemma\\2-27B}} & Vanilla & 5.12 & 3.45 & 8.95 & 6.55 & 4.67 & 2.33 & 7.38 & 8.63 & 7.26 & 8.41 & 3.13 & 7.70 & 8.34 & 1.31 & 13.55 & 1.76 & 0.99 & 5.86 & 3.31 \\
                    & Few Shot & 3.71 & 3.20 & 8.89 & 8.02 & 4.92 & 2.27 & 7.61 & 9.69 & 8.76 & 9.27 & 2.81 & 7.93 & 6.04 & 1.47 & 12.82 & 1.98 & 0.54 & 5.88 & 3.44 \\
                    & C-3PO & 4.06 & 3.90 & 8.15 & 6.23 & 5.27 & 2.72 & 8.41 & 8.79 & 6.49 & 8.70 & 2.85 & 8.02 & 8.66 & 1.57 & 12.44 & 1.47 & 1.41 & 5.83 & 3.15 \\
                    & Preference & 2.54 & 3.42 & 8.77 & 6.67 & 4.21 & 4.04 & 8.07 & 6.75 & 8.07 & 6.58 & 2.11 & 10.09 & 7.81 & 1.67 & 17.28 & 1.75 & 0.18 & 5.88 & 4.03 \\
                    \bottomrule
                \end{tabular}
            }
        }
        \caption{Country selection rate (\%) per model and method in the test set, except Preference, which reports the consensus answers C-3PO is trained to prefer (training split). All sixteen countries from BLEnD are included. ``XX'' indicates BLEnD's ``dummy'' option. The last two columns report the mean and standard deviation across countries.}
        \label{tab:cultural_bias_expanded}
    \end{table*}
}

\newcommand{\tabCThreePOImplications}{
    \begin{table*}[t]
        \centering
        \small
        \resizebox{\linewidth}{!}{
            {
                \begin{tabular}{l c@{\hspace{0.4em}}ccc c@{\hspace{0.4em}}ccc c@{\hspace{0.4em}}cc c@{\hspace{0.4em}}cc}
                    \toprule
                    \textbf{Model} & (i) & \textbf{Vanilla} & \textbf{Few-shot} & \textbf{C-3PO} & (ii) & \textbf{Persona} & \textbf{Steering} & \textbf{C-3PO} & (iii) & \textbf{Vanilla} & \textbf{C-3PO} & (iv) & \textbf{Vanilla} & \textbf{C-3PO} \\
                    \midrule
                    Qwen-2.5-3B & & 5.87 \scriptsize{(2.44)} & 5.88 \scriptsize{(2.51)} & 5.86 \scriptsize{(2.64)} & & 42.4 & 28.6 & 43.2 & & 62.7 & 62.4 & & 59.5 & 59.8 \\
                    Llama-3.1-8B & & 5.86 \scriptsize{(2.82)} & 5.88 \scriptsize{(2.56)} & 5.77 \scriptsize{(2.84)} & & 49.9 & 30.7 & 43.1 & & 69.8 & 62.0 & & 62.2 & 56.7 \\
                    Gemma-2-27B & & 5.86 \scriptsize{(3.31)} & 5.88 \scriptsize{(3.44)} & 5.83 \scriptsize{(3.15)} & & 58.4 & 31.5 & 55.9 & & 72.7 & 71.1 & & 65.7 & 63.3 \\
                    \bottomrule
                \end{tabular}
            }
        }
        \caption{(i) \textbf{Cultural Bias:} country selection rate (\%), averaged across countries (standard deviation in parentheses). (ii) \textbf{Persona Adherence:} accuracy (\%) of persona-country matches. (iii-iv) \textbf{Cultural Knowledge:} BLEnD accuracy (\%) on countries in (iii) and outside (iv) our dataset.}
        \label{tab:c3po_implications}
    \end{table*}
}

\newcommand{\tabHyperparameters}{
    \begin{table*}[htb]
        \centering
        \begin{tabular}{llc}
            \toprule
            \textbf{Method} & \textbf{Hyperparameter} & \textbf{Value(s)} \\
            \midrule
            \textbf{LoRA (C-3PO)} & Rank ($r$) & $\{8, 16, 32\}$ \\
            & Alpha ($\alpha$) & $2 \times r$ \\
            & Dropout & $0.05$ \\
            & Target Modules & \texttt{q, v, k, o, gate, up, down} \\
            \midrule
            \textbf{DPO (C-3PO)} & Epochs & $\{1, 3, 5\}$ \\
            & Learning Rate & $10^{-5}$ \\
            & Batch Size (per device) & 8 \\
            & Grad. Accumulation Steps & 4 \\
            & Warmup Ratio & $0.1$ \\
            & Beta ($\beta$) & $0.1$ \\
            & Max Length & 1024 \\
            & Max Prompt Length & 512 \\
            \midrule
            \textbf{Vector Steering} & Steering Coefficient ($\alpha$) & $\{0.01, 0.05, 0.1, 0.2, 0.5, 1.0\}$ \\
            & Layer Subsets & Early, Middle, Late, All, 20--29\textsuperscript{*} \\
            \midrule
            \textbf{Few-shot Prompt.} & Exemplars & 16 (8 languages $\times$ 2 samples) \\
            \bottomrule
            \multicolumn{3}{l}{\rule{0pt}{3ex}\footnotesize \textsuperscript{*}Layer ranges vary by model architecture.} \\
            \multicolumn{3}{l}{\footnotesize \textbf{\texttt{Llama-3.1-8B-Instruct}}: Early (5--14), Middle (15--24), Late (25--31), All (0--31).} \\
            \multicolumn{3}{l}{\footnotesize \textbf{\texttt{Gemma-2-27b-it}}: Early (5--17), Middle (18--31), Late (32--45), All (0--45).} \\
            \multicolumn{3}{l}{\footnotesize \textbf{\texttt{Qwen2.5-3B-Instruct}}: Early (5--14), Middle (15--24), Late (25--35), All (0--35).} \\
        \end{tabular}%
        \caption{Summary of hyperparameters used across all mitigation and fine-tuning experiments. All inference was performed with greedy decoding (\texttt{do\_sample=False}).}
        \label{tab:hyperparameters}
    \end{table*}
}

\newcommand{\tabTrainingData}{
    \begin{table}[htb]
        \centering
        \small
        \setlength{\tabcolsep}{4pt}
        \begin{tabular}{lccc}
            \toprule
            & \textbf{Qwen} & \textbf{Llama} & \textbf{Gemma} \\
            \midrule
            LoRA rank ($r$) & 8 & 32 & 32 \\
            Epochs & 5 & 5 & 3 \\
            \midrule
            Queries with a consensus & 1{,}242 & 1{,}265 & 1{,}326 \\
            Retained queries & 1{,}056 & 989 & 1{,}140 \\
            \quad fully consistent & 112 & 248 & 525 \\
            Preference pairs & 8{,}448 & 7{,}912 & 9{,}120 \\
            \midrule
            \multicolumn{4}{l}{\textit{Pairs from divergent languages}} \\
            English & 281 & 245 & 162 \\
            Chinese & 281 & 250 & 191 \\
            Spanish & 278 & 222 & 162 \\
            Indonesian & 309 & 246 & 187 \\
            Persian & 411 & 273 & 193 \\
            Korean & 327 & 258 & 206 \\
            Arabic & 344 & 257 & 185 \\
            Greek & 502 & 263 & 201 \\
            \bottomrule
        \end{tabular}
        \caption{Selected C-3PO configuration and training data (Qwen-2.5-3B, Llama-3.1-8B and Gemma-2-27B), from 1,394 training queries. Each retained query contributes one preference pair per language; a pair comes from a divergent language when that language's answer differed from the consensus or was invalid.}
        \label{tab:training_data}
    \end{table}
}

\newcommand{\tabOpenEndedLevels}{
    \begin{table}[htb]
        \centering
        \small
        \setlength{\tabcolsep}{3.5pt}
        \begin{tabular}{lccccc}
            \toprule
            & L0 & L1 & L2 & L3 & L4 \\
            \midrule
            Qwen-2.5-3B & +.003 & +.005 & +.006 & +.006 & +.010 \\
            Llama-3.1-8B & +.009 & +.018 & +.019 & +.022 & +.029$^{*}$ \\
            Gemma-2-27B & $-$.009 & +.032 & +.032 & +.037 & +.047$^{*}$ \\
            \bottomrule
        \end{tabular}
        \caption{Open-ended gain of C-3PO Vanilla over Vanilla in $\kappa_S$, over all eight languages, at each scoring level. $^{*}$: 95\% interval excludes zero, bootstrap over base questions.}
        \label{tab:open_ended_levels}
    \end{table}
}

\newcommand{\tabOpenEndedOptionShape}{
    \begin{table}[htb]
        \centering
        \small
        \setlength{\tabcolsep}{3pt}
        \begin{tabular}{llccccc}
            \toprule
            & & \multicolumn{3}{c}{\textbf{Lexicon match (\%)}} & \textbf{Identity} & \textbf{Length} \\
            \cmidrule(lr){3-5}
            & & Own & Other & Total & \textbf{(\%)} & \textbf{(chars)} \\
            \midrule
            \multirow{2}{*}{\makecell[l]{Qwen\\2.5-3B}} & Vanilla & 26.3 & 12.3 & 38.6 & 54.0 & 19.1 \\
            & C-3PO & 25.8 & 13.1 & 38.9 & 55.2 & 18.3 \\
            \midrule
            \multirow{2}{*}{\makecell[l]{Llama\\3.1-8B}} & Vanilla & 34.0 & 14.6 & 48.6 & 41.5 & 27.7 \\
            & C-3PO & 32.8 & 11.2 & 44.0 & 45.2 & 34.3 \\
            \midrule
            \multirow{2}{*}{\makecell[l]{Gemma\\2-27B}} & Vanilla & 38.1 & 15.7 & 53.8 & 39.4 & 15.0 \\
            & C-3PO & 36.7 & 13.1 & 49.8 & 42.7 & 13.7 \\
            \bottomrule
        \end{tabular}
        \caption{Shape of open-ended answers at L4, without personas: share matched to the question's own options or to other questions' options, share left to identity matching, and mean length. C-3PO does not make answers more option-like.}
        \label{tab:open_ended_option_shape}
    \end{table}
}

\newcommand{\tabOpenEndedComparison}{
    \begin{table}[htb]
        \centering
        \small
        \setlength{\tabcolsep}{3pt}
        \begin{tabular}{llccccc}
            \toprule
            & & \multicolumn{2}{c}{\textbf{Vanilla}} & \multicolumn{2}{c}{\textbf{$\kappa_S$}} & \textbf{Rel.} \\
            \cmidrule(lr){3-4} \cmidrule(lr){5-6}
            & & $P_e$ & $P_o$ & Vanilla & C-3PO & \textbf{gain} \\
            \midrule
            \multirow{2}{*}{\makecell[l]{Qwen\\2.5-3B}} & MCQ & .250 & .482 & .309 & .406 & +31\% \\
            & Open & .004 & .072 & .069 & .079 & +15\% \\
            \midrule
            \multirow{2}{*}{\makecell[l]{Llama\\3.1-8B}} & MCQ & .257 & .561 & .409 & .484 & +18\% \\
            & Open & .004 & .115 & .111 & .140 & +26\% \\
            \midrule
            \multirow{2}{*}{\makecell[l]{Gemma\\2-27B}} & MCQ & .250 & .686 & .581 & .658 & +13\% \\
            & Open & .006 & .241 & .237 & .283 & +20\% \\
            \bottomrule
        \end{tabular}
        \caption{Multiple-choice (MCQ) and open-ended consistency over all eight languages. Chance agreement $P_e$ nearly vanishes in the open-ended setting, so $\kappa_S$ approaches observed agreement $P_o$; the relative gain of C-3PO Vanilla over Vanilla is comparable across settings.}
        \label{tab:open_ended_comparison}
    \end{table}
}

\newcommand{\tabConsensusPairs}{
    \begin{table}[htb]
        \centering
        \small
        \setlength{\tabcolsep}{2pt}
        \begin{tabular}{lcccccc}
            \toprule
            & \multicolumn{2}{c}{\textbf{Qwen}} & \multicolumn{2}{c}{\textbf{Llama}} & \multicolumn{2}{c}{\textbf{Gemma}} \\
            \cmidrule(lr){2-3} \cmidrule(lr){4-5} \cmidrule(lr){6-7}
            \textbf{Consensus} & Disc. & Pairs & Disc. & Pairs & Disc. & Pairs \\
            \midrule
            C-3PO (plurality) & 141 & 8{,}392 & 120 & 8{,}504 & 76 & 8{,}704 \\
            Strict majority & 465 & 5{,}384 & 351 & 6{,}344 & 197 & 7{,}624 \\
            Soft vote & 0 & 9{,}712 & 0 & 9{,}584 & 0 & 9{,}680 \\
            \quad + confidence & 0 & 9{,}800 & 0 & 9{,}568 & 0 & 9{,}576 \\
            \quad + hard negatives & 0 & 9{,}688 & 0 & 9{,}504 & 0 & 9{,}624 \\
            \quad + margin gate & 278 & 7{,}072 & 246 & 7{,}448 & 141 & 8{,}144 \\
            \bottomrule
        \end{tabular}
        \caption{Training data of each consensus rule: training queries discarded for lack of a consensus or by the margin gate (Disc., of 1,394), and preference pairs retained after language-balanced undersampling (eight per retained query).}
        \label{tab:consensus_pairs}
    \end{table}
}

\newcommand{\tabAppendixSoftConsistency}{
    \begin{table*}[htb]
        \centering
        \resizebox{\linewidth}{!}{
            \begin{tabular}{cccccccccccccc}
                \toprule
                \multirow{2}{*}{\textbf{Model}} & \multirow{2}{*}{\makecell{\textbf{Language}\\\textbf{Group}}} & \multirow{2}{*}{\textbf{Vanilla}} & \multirow{2}{*}{\makecell{\textbf{Few}\\\textbf{shot}}} & \multicolumn{3}{c}{\textbf{Steering}} & \multicolumn{3}{c}{\textbf{Persona}} & \multirow{2}{*}{\makecell{\textbf{C-3PO} \\ \textbf{Vanilla}}} & \multicolumn{3}{c}{\textbf{C-3PO Persona}} \\
                \cmidrule(lr){5-7} \cmidrule(lr){8-10} \cmidrule(lr){12-14}
                & & & & \textbf{Min} & \textbf{Avg} & \textbf{Max} & \textbf{Min} & \textbf{Avg} & \textbf{Max} & & \textbf{Min} & \textbf{Avg} & \textbf{Max} \\
                \midrule
                \multirow[c]{5}{*}{\makecell{Qwen \\ 2.5-3B}} & All & {\cellcolor[HTML]{EAF3FB}} \color{black} .482 & {\cellcolor[HTML]{3C8CC3}} \color{black} .541 & {\cellcolor[HTML]{F1F7FD}} \color{black} .478 \,{\footnotesize CN} & {\cellcolor[HTML]{EBF3FB}} \color{black} .481 & {\cellcolor[HTML]{E8F1FA}} \color{black} .483 \,{\footnotesize MX} & {\cellcolor[HTML]{F7FBFF}} \color{black} .475 \,{\footnotesize ID} & {\cellcolor[HTML]{CEE0F2}} \color{black} .496 & {\cellcolor[HTML]{94C4DF}} \color{black} .516 \,{\footnotesize US} & {\cellcolor[HTML]{1D6CB1}} \color{white} .553 & {\cellcolor[HTML]{3686C0}} \color{white} .543 \,{\footnotesize CN} & {\cellcolor[HTML]{0E59A2}} \color{white} .560 & {\cellcolor[HTML]{08306B}} \color{white} .577 \,{\footnotesize US} \\
                & High-Res & {\cellcolor[HTML]{ADD0E6}} \color{black} .595 & {\cellcolor[HTML]{4695C8}} \color{black} .632 & {\cellcolor[HTML]{BAD6EB}} \color{black} .589 \,{\footnotesize KR} & {\cellcolor[HTML]{A9CFE5}} \color{black} .596 & {\cellcolor[HTML]{9AC8E0}} \color{black} .602 \,{\footnotesize GR} & {\cellcolor[HTML]{F7FBFF}} \color{black} .552 \,{\footnotesize ID} & {\cellcolor[HTML]{C7DCEF}} \color{black} .584 & {\cellcolor[HTML]{A4CCE3}} \color{black} .598 \,{\footnotesize DZ} & {\cellcolor[HTML]{08306B}} \color{white} .683 & {\cellcolor[HTML]{69ADD5}} \color{black} .618 \,{\footnotesize ID} & {\cellcolor[HTML]{3181BD}} \color{white} .642 & {\cellcolor[HTML]{135FA7}} \color{white} .659 \,{\footnotesize DZ} \\
                & Low-Res & {\cellcolor[HTML]{F4F9FE}} \color{black} .455 & {\cellcolor[HTML]{56A0CE}} \color{black} .510 & {\cellcolor[HTML]{F7FBFF}} \color{black} .453 \,{\footnotesize GR} & {\cellcolor[HTML]{F0F6FD}} \color{black} .457 & {\cellcolor[HTML]{E9F2FA}} \color{black} .460 \,{\footnotesize MX} & {\cellcolor[HTML]{F5FAFE}} \color{black} .454 \,{\footnotesize CN} & {\cellcolor[HTML]{CDDFF1}} \color{black} .475 & {\cellcolor[HTML]{85BCDC}} \color{black} .496 \,{\footnotesize US} & {\cellcolor[HTML]{2D7DBB}} \color{white} .524 & {\cellcolor[HTML]{4493C7}} \color{black} .515 \,{\footnotesize CN} & {\cellcolor[HTML]{0E59A2}} \color{white} .538 & {\cellcolor[HTML]{08306B}} \color{white} .553 \,{\footnotesize US} \\
                & Indo-Eu & {\cellcolor[HTML]{F2F7FD}} \color{black} .435 & {\cellcolor[HTML]{3787C0}} \color{white} .497 & {\cellcolor[HTML]{F7FBFF}} \color{black} .432 \,{\footnotesize US} & {\cellcolor[HTML]{E9F2FA}} \color{black} .439 & {\cellcolor[HTML]{DFEBF7}} \color{black} .444 \,{\footnotesize GR} & {\cellcolor[HTML]{E7F0FA}} \color{black} .440 \,{\footnotesize ID} & {\cellcolor[HTML]{BAD6EB}} \color{black} .460 & {\cellcolor[HTML]{7CB7DA}} \color{black} .477 \,{\footnotesize IR} & {\cellcolor[HTML]{1764AB}} \color{white} .509 & {\cellcolor[HTML]{4B98CA}} \color{black} .490 \,{\footnotesize CN} & {\cellcolor[HTML]{125EA6}} \color{white} .512 & {\cellcolor[HTML]{08306B}} \color{white} .529 \,{\footnotesize DZ} \\
                & Non-Indo & {\cellcolor[HTML]{ECF4FB}} \color{black} .515 & {\cellcolor[HTML]{82BBDB}} \color{black} .566 & {\cellcolor[HTML]{F7FBFF}} \color{black} .508 \,{\footnotesize KR} & {\cellcolor[HTML]{EFF6FC}} \color{black} .514 & {\cellcolor[HTML]{E7F1FA}} \color{black} .519 \,{\footnotesize IR} & {\cellcolor[HTML]{ECF4FB}} \color{black} .515 \,{\footnotesize CN} & {\cellcolor[HTML]{D2E3F3}} \color{black} .533 & {\cellcolor[HTML]{A6CEE4}} \color{black} .554 \,{\footnotesize DZ} & {\cellcolor[HTML]{4292C6}} \color{black} .590 & {\cellcolor[HTML]{4594C7}} \color{black} .589 \,{\footnotesize ID} & {\cellcolor[HTML]{2171B5}} \color{white} .606 & {\cellcolor[HTML]{08306B}} \color{white} .640 \,{\footnotesize US} \\
                \midrule
                \multirow[c]{5}{*}{\makecell{Llama \\ 3.1-8B}} & All & {\cellcolor[HTML]{E3EEF8}} \color{black} .561 & {\cellcolor[HTML]{084184}} \color{white} .626 & {\cellcolor[HTML]{E9F2FA}} \color{black} .558 \,{\footnotesize CN} & {\cellcolor[HTML]{DCEAF6}} \color{black} .563 & {\cellcolor[HTML]{CFE1F2}} \color{black} .569 \,{\footnotesize US} & {\cellcolor[HTML]{F7FBFF}} \color{black} .553 \,{\footnotesize MX} & {\cellcolor[HTML]{CADEF0}} \color{black} .571 & {\cellcolor[HTML]{4090C5}} \color{black} .602 \,{\footnotesize US} & {\cellcolor[HTML]{1966AD}} \color{white} .615 & {\cellcolor[HTML]{A1CBE2}} \color{black} .581 \,{\footnotesize ID} & {\cellcolor[HTML]{3585BF}} \color{white} .605 & {\cellcolor[HTML]{08306B}} \color{white} .631 \,{\footnotesize US} \\
                & High-Res & {\cellcolor[HTML]{B8D5EA}} \color{black} .604 & {\cellcolor[HTML]{1663AA}} \color{white} .665 & {\cellcolor[HTML]{CADDF0}} \color{black} .596 \,{\footnotesize IR} & {\cellcolor[HTML]{B9D6EA}} \color{black} .603 & {\cellcolor[HTML]{A8CEE4}} \color{black} .610 \,{\footnotesize CN} & {\cellcolor[HTML]{F7FBFF}} \color{black} .568 \,{\footnotesize ID} & {\cellcolor[HTML]{BED8EC}} \color{black} .601 & {\cellcolor[HTML]{1A68AE}} \color{white} .662 \,{\footnotesize US} & {\cellcolor[HTML]{3686C0}} \color{white} .649 & {\cellcolor[HTML]{CDDFF1}} \color{black} .594 \,{\footnotesize ID} & {\cellcolor[HTML]{6CAED6}} \color{black} .628 & {\cellcolor[HTML]{08306B}} \color{white} .689 \,{\footnotesize US} \\
                & Low-Res & {\cellcolor[HTML]{E1EDF8}} \color{black} .549 & {\cellcolor[HTML]{084F99}} \color{white} .611 & {\cellcolor[HTML]{EDF4FC}} \color{black} .545 \,{\footnotesize CN} & {\cellcolor[HTML]{D7E6F5}} \color{black} .553 & {\cellcolor[HTML]{CDE0F1}} \color{black} .558 \,{\footnotesize US} & {\cellcolor[HTML]{F7FBFF}} \color{black} .540 \,{\footnotesize MX} & {\cellcolor[HTML]{C4DAEE}} \color{black} .561 & {\cellcolor[HTML]{4896C8}} \color{black} .589 \,{\footnotesize US} & {\cellcolor[HTML]{1967AD}} \color{white} .604 & {\cellcolor[HTML]{74B3D8}} \color{black} .579 \,{\footnotesize ID} & {\cellcolor[HTML]{2474B7}} \color{white} .600 & {\cellcolor[HTML]{08306B}} \color{white} .621 \,{\footnotesize IR} \\
                & Indo-Eu & {\cellcolor[HTML]{E8F1FA}} \color{black} .572 & {\cellcolor[HTML]{2C7CBA}} \color{white} .625 & {\cellcolor[HTML]{F2F8FD}} \color{black} .568 \,{\footnotesize CN} & {\cellcolor[HTML]{DDEAF7}} \color{black} .577 & {\cellcolor[HTML]{C9DDF0}} \color{black} .586 \,{\footnotesize US} & {\cellcolor[HTML]{F7FBFF}} \color{black} .566 \,{\footnotesize CN} & {\cellcolor[HTML]{BED8EC}} \color{black} .589 & {\cellcolor[HTML]{2575B7}} \color{white} .627 \,{\footnotesize GR} & {\cellcolor[HTML]{3787C0}} \color{white} .622 & {\cellcolor[HTML]{D9E8F5}} \color{black} .578 \,{\footnotesize CN} & {\cellcolor[HTML]{4191C6}} \color{black} .618 & {\cellcolor[HTML]{08306B}} \color{white} .650 \,{\footnotesize IR} \\
                & Non-Indo & {\cellcolor[HTML]{C3DAEE}} \color{black} .546 & {\cellcolor[HTML]{08519C}} \color{white} .617 & {\cellcolor[HTML]{D1E2F3}} \color{black} .538 \,{\footnotesize DZ} & {\cellcolor[HTML]{C9DDF0}} \color{black} .543 & {\cellcolor[HTML]{C6DBEF}} \color{black} .545 \,{\footnotesize GR} & {\cellcolor[HTML]{F7FBFF}} \color{black} .516 \,{\footnotesize DZ} & {\cellcolor[HTML]{CCDFF1}} \color{black} .542 & {\cellcolor[HTML]{529DCC}} \color{black} .583 \,{\footnotesize US} & {\cellcolor[HTML]{3181BD}} \color{white} .595 & {\cellcolor[HTML]{BAD6EB}} \color{black} .549 \,{\footnotesize KR} & {\cellcolor[HTML]{4E9ACB}} \color{black} .584 & {\cellcolor[HTML]{08306B}} \color{white} .631 \,{\footnotesize US} \\
                \midrule
                \multirow[c]{5}{*}{\makecell{Gemma \\ 2-27B}} & All & {\cellcolor[HTML]{EDF4FC}} \color{black} .686 & {\cellcolor[HTML]{F2F8FD}} \color{black} .684 & {\cellcolor[HTML]{F7FBFF}} \color{black} .682 \,{\footnotesize DZ} & {\cellcolor[HTML]{EEF5FC}} \color{black} .686 & {\cellcolor[HTML]{E3EEF9}} \color{black} .690 \,{\footnotesize US} & {\cellcolor[HTML]{D6E6F4}} \color{black} .695 \,{\footnotesize MX} & {\cellcolor[HTML]{97C6DF}} \color{black} .714 & {\cellcolor[HTML]{4997C9}} \color{black} .731 \,{\footnotesize KR} & {\cellcolor[HTML]{2070B4}} \color{white} .744 & {\cellcolor[HTML]{539ECD}} \color{black} .729 \,{\footnotesize MX} & {\cellcolor[HTML]{135FA7}} \color{white} .749 & {\cellcolor[HTML]{08306B}} \color{white} .764 \,{\footnotesize IR} \\
                & High-Res & {\cellcolor[HTML]{D0E1F2}} \color{black} .718 & {\cellcolor[HTML]{BCD7EB}} \color{black} .725 & {\cellcolor[HTML]{F7FBFF}} \color{black} .699 \,{\footnotesize DZ} & {\cellcolor[HTML]{EAF2FB}} \color{black} .706 & {\cellcolor[HTML]{E1EDF8}} \color{black} .709 \,{\footnotesize US} & {\cellcolor[HTML]{D6E6F4}} \color{black} .714 \,{\footnotesize MX} & {\cellcolor[HTML]{7DB8DA}} \color{black} .741 & {\cellcolor[HTML]{2F7FBC}} \color{white} .764 \,{\footnotesize IR} & {\cellcolor[HTML]{6DAFD7}} \color{black} .745 & {\cellcolor[HTML]{ADD0E6}} \color{black} .730 \,{\footnotesize MX} & {\cellcolor[HTML]{1B69AF}} \color{white} .771 & {\cellcolor[HTML]{08306B}} \color{white} .792 \,{\footnotesize US} \\
                & Low-Res & {\cellcolor[HTML]{F3F8FE}} \color{black} .674 & {\cellcolor[HTML]{F5FAFE}} \color{black} .673 & {\cellcolor[HTML]{F7FBFF}} \color{black} .672 \,{\footnotesize IR} & {\cellcolor[HTML]{EAF3FB}} \color{black} .678 & {\cellcolor[HTML]{E1EDF8}} \color{black} .682 \,{\footnotesize CN} & {\cellcolor[HTML]{E3EEF9}} \color{black} .681 \,{\footnotesize ID} & {\cellcolor[HTML]{AACFE5}} \color{black} .703 & {\cellcolor[HTML]{2A7AB9}} \color{white} .737 \,{\footnotesize KR} & {\cellcolor[HTML]{1967AD}} \color{white} .743 & {\cellcolor[HTML]{6FB0D7}} \color{black} .716 \,{\footnotesize ID} & {\cellcolor[HTML]{1865AC}} \color{white} .744 & {\cellcolor[HTML]{08306B}} \color{white} .762 \,{\footnotesize KR} \\
                & Indo-Eu & {\cellcolor[HTML]{CADDF0}} \color{black} .720 & {\cellcolor[HTML]{F7FBFF}} \color{black} .699 & {\cellcolor[HTML]{D5E5F4}} \color{black} .714 \,{\footnotesize GR} & {\cellcolor[HTML]{CDDFF1}} \color{black} .718 & {\cellcolor[HTML]{C8DCF0}} \color{black} .720 \,{\footnotesize US} & {\cellcolor[HTML]{CCDFF1}} \color{black} .718 \,{\footnotesize MX} & {\cellcolor[HTML]{7AB6D9}} \color{black} .739 & {\cellcolor[HTML]{3080BD}} \color{white} .759 \,{\footnotesize US} & {\cellcolor[HTML]{3C8CC3}} \color{black} .755 & {\cellcolor[HTML]{4594C7}} \color{black} .752 \,{\footnotesize ID} & {\cellcolor[HTML]{115CA5}} \color{white} .771 & {\cellcolor[HTML]{08306B}} \color{white} .785 \,{\footnotesize IR} \\
                & Non-Indo & {\cellcolor[HTML]{F7FBFF}} \color{black} .656 & {\cellcolor[HTML]{D9E7F5}} \color{black} .670 & {\cellcolor[HTML]{E8F1FA}} \color{black} .662 \,{\footnotesize ID} & {\cellcolor[HTML]{DEEBF7}} \color{black} .667 & {\cellcolor[HTML]{D1E2F3}} \color{black} .673 \,{\footnotesize US} & {\cellcolor[HTML]{D7E6F5}} \color{black} .670 \,{\footnotesize ID} & {\cellcolor[HTML]{A1CBE2}} \color{black} .689 & {\cellcolor[HTML]{3484BF}} \color{white} .717 \,{\footnotesize KR} & {\cellcolor[HTML]{0C56A0}} \color{white} .733 & {\cellcolor[HTML]{4292C6}} \color{black} .712 \,{\footnotesize MX} & {\cellcolor[HTML]{115CA5}} \color{white} .731 & {\cellcolor[HTML]{08306B}} \color{white} .746 \,{\footnotesize GR} \\
                \bottomrule
            \end{tabular}}
        \caption{Soft Consistency ($\uparrow$) across models, language groups and mitigation strategies, in the format of \autoref{tab:main_results}: Steering, Persona Prompting and C-3PO Persona report the minimum, average and maximum over the eight personas (indicated by country code), and colour scales are applied row-wise.}
        \label{tab:appendix_soft_consistency}
    \end{table*}
}

\newcommand{\tabAppendixHardConsistency}{
    \begin{table*}[htb]
        \centering
        \resizebox{\linewidth}{!}{
            \begin{tabular}{cccccccccccccc}
                \toprule
                \multirow{2}{*}{\textbf{Model}} & \multirow{2}{*}{\makecell{\textbf{Language}\\\textbf{Group}}} & \multirow{2}{*}{\textbf{Vanilla}} & \multirow{2}{*}{\makecell{\textbf{Few}\\\textbf{shot}}} & \multicolumn{3}{c}{\textbf{Steering}} & \multicolumn{3}{c}{\textbf{Persona}} & \multirow{2}{*}{\makecell{\textbf{C-3PO} \\ \textbf{Vanilla}}} & \multicolumn{3}{c}{\textbf{C-3PO Persona}} \\
                \cmidrule(lr){5-7} \cmidrule(lr){8-10} \cmidrule(lr){12-14}
                & & & & \textbf{Min} & \textbf{Avg} & \textbf{Max} & \textbf{Min} & \textbf{Avg} & \textbf{Max} & & \textbf{Min} & \textbf{Avg} & \textbf{Max} \\
                \midrule
                \multirow[c]{5}{*}{\makecell{Qwen \\ 2.5-3B}} & All & {\cellcolor[HTML]{D9E8F5}} \color{black} .092 & {\cellcolor[HTML]{3B8BC2}} \color{white} .151 & {\cellcolor[HTML]{D9E8F5}} \color{black} .092 \,{\footnotesize GR} & {\cellcolor[HTML]{D1E2F3}} \color{black} .097 & {\cellcolor[HTML]{CDDFF1}} \color{black} .100 \,{\footnotesize US} & {\cellcolor[HTML]{F7FBFF}} \color{black} .074 \,{\footnotesize ID} & {\cellcolor[HTML]{BDD7EC}} \color{black} .107 & {\cellcolor[HTML]{5CA4D0}} \color{black} .138 \,{\footnotesize US} & {\cellcolor[HTML]{4896C8}} \color{black} .146 & {\cellcolor[HTML]{85BCDC}} \color{black} .125 \,{\footnotesize CN} & {\cellcolor[HTML]{3686C0}} \color{white} .153 & {\cellcolor[HTML]{08306B}} \color{white} .192 \,{\footnotesize US} \\
                & High-Res & {\cellcolor[HTML]{9DCAE1}} \color{black} .442 & {\cellcolor[HTML]{3383BE}} \color{white} .496 & {\cellcolor[HTML]{A6CEE4}} \color{black} .437 \,{\footnotesize KR} & {\cellcolor[HTML]{97C6DF}} \color{black} .445 & {\cellcolor[HTML]{8CC0DD}} \color{black} .450 \,{\footnotesize GR} & {\cellcolor[HTML]{F7FBFF}} \color{black} .376 \,{\footnotesize ID} & {\cellcolor[HTML]{B9D6EA}} \color{black} .427 & {\cellcolor[HTML]{85BCDC}} \color{black} .453 \,{\footnotesize DZ} & {\cellcolor[HTML]{08306B}} \color{white} .552 & {\cellcolor[HTML]{74B3D8}} \color{black} .460 \,{\footnotesize ID} & {\cellcolor[HTML]{3282BE}} \color{white} .497 & {\cellcolor[HTML]{0E59A2}} \color{white} .524 \,{\footnotesize US} \\
                & Low-Res & {\cellcolor[HTML]{E9F2FA}} \color{black} .146 & {\cellcolor[HTML]{3383BE}} \color{white} .210 & {\cellcolor[HTML]{F7FBFF}} \color{black} .138 \,{\footnotesize CN} & {\cellcolor[HTML]{E8F1FA}} \color{black} .146 & {\cellcolor[HTML]{DAE8F6}} \color{black} .153 \,{\footnotesize US} & {\cellcolor[HTML]{F2F8FD}} \color{black} .141 \,{\footnotesize CN} & {\cellcolor[HTML]{CCDFF1}} \color{black} .161 & {\cellcolor[HTML]{8DC1DD}} \color{black} .182 \,{\footnotesize US} & {\cellcolor[HTML]{115CA5}} \color{white} .225 & {\cellcolor[HTML]{7AB6D9}} \color{black} .187 \,{\footnotesize CN} & {\cellcolor[HTML]{1B69AF}} \color{white} .220 & {\cellcolor[HTML]{08306B}} \color{white} .243 \,{\footnotesize US} \\
                & Indo-Eu & {\cellcolor[HTML]{F2F8FD}} \color{black} .161 & {\cellcolor[HTML]{3181BD}} \color{white} .233 & {\cellcolor[HTML]{F7FBFF}} \color{black} .159 \,{\footnotesize CN} & {\cellcolor[HTML]{EAF2FB}} \color{black} .166 & {\cellcolor[HTML]{DBE9F6}} \color{black} .174 \,{\footnotesize GR} & {\cellcolor[HTML]{E9F2FA}} \color{black} .166 \,{\footnotesize ID} & {\cellcolor[HTML]{BAD6EB}} \color{black} .189 & {\cellcolor[HTML]{3D8DC4}} \color{black} .228 \,{\footnotesize US} & {\cellcolor[HTML]{2474B7}} \color{white} .238 & {\cellcolor[HTML]{89BEDC}} \color{black} .205 \,{\footnotesize CN} & {\cellcolor[HTML]{1E6DB2}} \color{white} .241 & {\cellcolor[HTML]{08306B}} \color{white} .266 \,{\footnotesize DZ} \\
                & Non-Indo & {\cellcolor[HTML]{E7F1FA}} \color{black} .263 & {\cellcolor[HTML]{8CC0DD}} \color{black} .327 & {\cellcolor[HTML]{F5F9FE}} \color{black} .251 \,{\footnotesize CN} & {\cellcolor[HTML]{EBF3FB}} \color{black} .259 & {\cellcolor[HTML]{E4EFF9}} \color{black} .266 \,{\footnotesize IR} & {\cellcolor[HTML]{F7FBFF}} \color{black} .248 \,{\footnotesize CN} & {\cellcolor[HTML]{D7E6F5}} \color{black} .278 & {\cellcolor[HTML]{AED1E7}} \color{black} .309 \,{\footnotesize US} & {\cellcolor[HTML]{66ABD4}} \color{black} .345 & {\cellcolor[HTML]{66ABD4}} \color{black} .345 \,{\footnotesize ID} & {\cellcolor[HTML]{3484BF}} \color{white} .377 & {\cellcolor[HTML]{08306B}} \color{white} .437 \,{\footnotesize US} \\
                \midrule
                \multirow[c]{5}{*}{\makecell{Llama \\ 3.1-8B}} & All & {\cellcolor[HTML]{C7DCEF}} \color{black} .182 & {\cellcolor[HTML]{61A7D2}} \color{black} .220 & {\cellcolor[HTML]{CEE0F2}} \color{black} .176 \,{\footnotesize US} & {\cellcolor[HTML]{C9DDF0}} \color{black} .180 & {\cellcolor[HTML]{C2D9EE}} \color{black} .184 \,{\footnotesize IR} & {\cellcolor[HTML]{F7FBFF}} \color{black} .148 \,{\footnotesize ID} & {\cellcolor[HTML]{CADDF0}} \color{black} .180 & {\cellcolor[HTML]{7FB9DA}} \color{black} .210 \,{\footnotesize US} & {\cellcolor[HTML]{206FB4}} \color{white} .251 & {\cellcolor[HTML]{95C5DF}} \color{black} .202 \,{\footnotesize ID} & {\cellcolor[HTML]{3B8BC2}} \color{white} .237 & {\cellcolor[HTML]{08306B}} \color{white} .284 \,{\footnotesize US} \\
                & High-Res & {\cellcolor[HTML]{B5D4E9}} \color{black} .453 & {\cellcolor[HTML]{3383BE}} \color{white} .519 & {\cellcolor[HTML]{C7DBEF}} \color{black} .442 \,{\footnotesize IR} & {\cellcolor[HTML]{BAD6EB}} \color{black} .450 & {\cellcolor[HTML]{ABD0E6}} \color{black} .458 \,{\footnotesize CN} & {\cellcolor[HTML]{F7FBFF}} \color{black} .399 \,{\footnotesize ID} & {\cellcolor[HTML]{C2D9EE}} \color{black} .446 & {\cellcolor[HTML]{2F7FBC}} \color{white} .522 \,{\footnotesize US} & {\cellcolor[HTML]{3383BE}} \color{white} .519 & {\cellcolor[HTML]{B9D6EA}} \color{black} .450 \,{\footnotesize ID} & {\cellcolor[HTML]{64A9D3}} \color{black} .491 & {\cellcolor[HTML]{08306B}} \color{white} .575 \,{\footnotesize US} \\
                & Low-Res & {\cellcolor[HTML]{C4DAEE}} \color{black} .256 & {\cellcolor[HTML]{3989C1}} \color{white} .304 & {\cellcolor[HTML]{D6E5F4}} \color{black} .246 \,{\footnotesize CN} & {\cellcolor[HTML]{CBDEF1}} \color{black} .252 & {\cellcolor[HTML]{BED8EC}} \color{black} .258 \,{\footnotesize IR} & {\cellcolor[HTML]{F7FBFF}} \color{black} .225 \,{\footnotesize ID} & {\cellcolor[HTML]{C6DBEF}} \color{black} .255 & {\cellcolor[HTML]{6FB0D7}} \color{black} .284 \,{\footnotesize GR} & {\cellcolor[HTML]{1967AD}} \color{white} .320 & {\cellcolor[HTML]{8ABFDD}} \color{black} .276 \,{\footnotesize KR} & {\cellcolor[HTML]{2575B7}} \color{white} .313 & {\cellcolor[HTML]{08306B}} \color{white} .345 \,{\footnotesize IR} \\
                & Indo-Eu & {\cellcolor[HTML]{E5EFF9}} \color{black} .322 & {\cellcolor[HTML]{539ECD}} \color{black} .389 & {\cellcolor[HTML]{F7FBFF}} \color{black} .309 \,{\footnotesize CN} & {\cellcolor[HTML]{E3EEF8}} \color{black} .324 & {\cellcolor[HTML]{D6E6F4}} \color{black} .332 \,{\footnotesize MX} & {\cellcolor[HTML]{F4F9FE}} \color{black} .312 \,{\footnotesize ID} & {\cellcolor[HTML]{CADEF0}} \color{black} .341 & {\cellcolor[HTML]{3787C0}} \color{white} .402 \,{\footnotesize GR} & {\cellcolor[HTML]{4090C5}} \color{black} .396 & {\cellcolor[HTML]{C8DCF0}} \color{black} .343 \,{\footnotesize CN} & {\cellcolor[HTML]{3989C1}} \color{white} .400 & {\cellcolor[HTML]{08306B}} \color{white} .448 \,{\footnotesize IR} \\
                & Non-Indo & {\cellcolor[HTML]{CFE1F2}} \color{black} .289 & {\cellcolor[HTML]{2676B8}} \color{white} .373 & {\cellcolor[HTML]{D8E7F5}} \color{black} .281 \,{\footnotesize DZ} & {\cellcolor[HTML]{CDDFF1}} \color{black} .291 & {\cellcolor[HTML]{C4DAEE}} \color{black} .297 \,{\footnotesize GR} & {\cellcolor[HTML]{F7FBFF}} \color{black} .256 \,{\footnotesize ID} & {\cellcolor[HTML]{D6E6F4}} \color{black} .282 & {\cellcolor[HTML]{89BEDC}} \color{black} .325 \,{\footnotesize US} & {\cellcolor[HTML]{3383BE}} \color{white} .366 & {\cellcolor[HTML]{B0D2E7}} \color{black} .307 \,{\footnotesize ID} & {\cellcolor[HTML]{5CA4D0}} \color{black} .343 & {\cellcolor[HTML]{08306B}} \color{white} .417 \,{\footnotesize US} \\
                \midrule
                \multirow[c]{5}{*}{\makecell{Gemma \\ 2-27B}} & All & {\cellcolor[HTML]{DFEBF7}} \color{black} .355 & {\cellcolor[HTML]{DFEBF7}} \color{black} .355 & {\cellcolor[HTML]{DFEBF7}} \color{black} .355 \,{\footnotesize DZ} & {\cellcolor[HTML]{D2E3F3}} \color{black} .365 & {\cellcolor[HTML]{C7DCEF}} \color{black} .373 \,{\footnotesize US} & {\cellcolor[HTML]{F7FBFF}} \color{black} .338 \,{\footnotesize ID} & {\cellcolor[HTML]{A9CFE5}} \color{black} .387 & {\cellcolor[HTML]{4B98CA}} \color{black} .425 \,{\footnotesize US} & {\cellcolor[HTML]{083877}} \color{white} .478 & {\cellcolor[HTML]{5CA4D0}} \color{black} .417 \,{\footnotesize ID} & {\cellcolor[HTML]{1561A9}} \color{white} .456 & {\cellcolor[HTML]{08306B}} \color{white} .483 \,{\footnotesize IR} \\
                & High-Res & {\cellcolor[HTML]{CBDEF1}} \color{black} .601 & {\cellcolor[HTML]{BCD7EB}} \color{black} .609 & {\cellcolor[HTML]{F7FBFF}} \color{black} .570 \,{\footnotesize DZ} & {\cellcolor[HTML]{EEF5FC}} \color{black} .577 & {\cellcolor[HTML]{E4EFF9}} \color{black} .583 \,{\footnotesize US} & {\cellcolor[HTML]{D6E5F4}} \color{black} .593 \,{\footnotesize MX} & {\cellcolor[HTML]{7FB9DA}} \color{black} .631 & {\cellcolor[HTML]{2F7FBC}} \color{white} .665 \,{\footnotesize IR} & {\cellcolor[HTML]{5BA3D0}} \color{black} .645 & {\cellcolor[HTML]{9DCAE1}} \color{black} .621 \,{\footnotesize MX} & {\cellcolor[HTML]{1C6BB0}} \color{white} .675 & {\cellcolor[HTML]{08306B}} \color{white} .706 \,{\footnotesize US} \\
                & Low-Res & {\cellcolor[HTML]{E6F0F9}} \color{black} .425 & {\cellcolor[HTML]{F7FBFF}} \color{black} .412 & {\cellcolor[HTML]{DFECF7}} \color{black} .430 \,{\footnotesize ID} & {\cellcolor[HTML]{D5E5F4}} \color{black} .438 & {\cellcolor[HTML]{CFE1F2}} \color{black} .442 \,{\footnotesize CN} & {\cellcolor[HTML]{EDF4FC}} \color{black} .419 \,{\footnotesize ID} & {\cellcolor[HTML]{B5D4E9}} \color{black} .457 & {\cellcolor[HTML]{2D7DBB}} \color{white} .517 \,{\footnotesize KR} & {\cellcolor[HTML]{08478D}} \color{white} .547 & {\cellcolor[HTML]{6AAED6}} \color{black} .486 \,{\footnotesize ID} & {\cellcolor[HTML]{1A68AE}} \color{white} .528 & {\cellcolor[HTML]{08306B}} \color{white} .560 \,{\footnotesize KR} \\
                & Indo-Eu & {\cellcolor[HTML]{B5D4E9}} \color{black} .542 & {\cellcolor[HTML]{F7FBFF}} \color{black} .501 & {\cellcolor[HTML]{BCD7EB}} \color{black} .540 \,{\footnotesize DZ} & {\cellcolor[HTML]{B2D2E8}} \color{black} .544 & {\cellcolor[HTML]{A4CCE3}} \color{black} .550 \,{\footnotesize US} & {\cellcolor[HTML]{CBDEF1}} \color{black} .532 \,{\footnotesize DZ} & {\cellcolor[HTML]{75B4D8}} \color{black} .566 & {\cellcolor[HTML]{2474B7}} \color{white} .601 \,{\footnotesize US} & {\cellcolor[HTML]{2F7FBC}} \color{white} .596 & {\cellcolor[HTML]{56A0CE}} \color{black} .578 \,{\footnotesize ID} & {\cellcolor[HTML]{1460A8}} \color{white} .612 & {\cellcolor[HTML]{08306B}} \color{white} .637 \,{\footnotesize IR} \\
                & Non-Indo & {\cellcolor[HTML]{F7FBFF}} \color{black} .430 & {\cellcolor[HTML]{D5E5F4}} \color{black} .455 & {\cellcolor[HTML]{EAF2FB}} \color{black} .440 \,{\footnotesize DZ} & {\cellcolor[HTML]{D9E8F5}} \color{black} .452 & {\cellcolor[HTML]{CFE1F2}} \color{black} .460 \,{\footnotesize US} & {\cellcolor[HTML]{E6F0F9}} \color{black} .442 \,{\footnotesize ID} & {\cellcolor[HTML]{A9CFE5}} \color{black} .480 & {\cellcolor[HTML]{4997C9}} \color{black} .519 \,{\footnotesize KR} & {\cellcolor[HTML]{0A549E}} \color{white} .558 & {\cellcolor[HTML]{3585BF}} \color{white} .529 \,{\footnotesize MX} & {\cellcolor[HTML]{0D57A1}} \color{white} .556 & {\cellcolor[HTML]{08306B}} \color{white} .578 \,{\footnotesize GR} \\
                \bottomrule
            \end{tabular}}
        \caption{Hard Consistency ($\uparrow$) across models, language groups and mitigation strategies, in the format of \autoref{tab:main_results}: Steering, Persona Prompting and C-3PO Persona report the minimum, average and maximum over the eight personas (indicated by country code), and colour scales are applied row-wise.}
        \label{tab:appendix_hard_consistency}
    \end{table*}
}

\newcommand{\tabAppendixModeFrequency}{
    \begin{table*}[htb]
        \centering
        \resizebox{\linewidth}{!}{
            \begin{tabular}{cccccccccccccc}
                \toprule
                \multirow{2}{*}{\textbf{Model}} & \multirow{2}{*}{\makecell{\textbf{Language}\\\textbf{Group}}} & \multirow{2}{*}{\textbf{Vanilla}} & \multirow{2}{*}{\makecell{\textbf{Few}\\\textbf{shot}}} & \multicolumn{3}{c}{\textbf{Steering}} & \multicolumn{3}{c}{\textbf{Persona}} & \multirow{2}{*}{\makecell{\textbf{C-3PO} \\ \textbf{Vanilla}}} & \multicolumn{3}{c}{\textbf{C-3PO Persona}} \\
                \cmidrule(lr){5-7} \cmidrule(lr){8-10} \cmidrule(lr){12-14}
                & & & & \textbf{Min} & \textbf{Avg} & \textbf{Max} & \textbf{Min} & \textbf{Avg} & \textbf{Max} & & \textbf{Min} & \textbf{Avg} & \textbf{Max} \\
                \midrule
                \multirow[c]{5}{*}{\makecell{Qwen \\ 2.5-3B}} & All & {\cellcolor[HTML]{083B7C}} \color{white} .647 & {\cellcolor[HTML]{9AC8E0}} \color{white} .690 & {\cellcolor[HTML]{08306B}} \color{white} .644 \,{\footnotesize US} & {\cellcolor[HTML]{083A7A}} \color{white} .647 & {\cellcolor[HTML]{084184}} \color{white} .648 \,{\footnotesize MX} & {\cellcolor[HTML]{08316D}} \color{white} .644 \,{\footnotesize ID} & {\cellcolor[HTML]{1865AC}} \color{white} .659 & {\cellcolor[HTML]{4493C7}} \color{white} .672 \,{\footnotesize US} & {\cellcolor[HTML]{CDDFF1}} \color{black} .703 & {\cellcolor[HTML]{AED1E7}} \color{black} .695 \,{\footnotesize CN} & {\cellcolor[HTML]{DDEAF7}} \color{black} .709 & {\cellcolor[HTML]{F7FBFF}} \color{black} .719 \,{\footnotesize US} \\
                & High-Res & {\cellcolor[HTML]{2D7DBB}} \color{white} .781 & {\cellcolor[HTML]{81BADB}} \color{white} .800 & {\cellcolor[HTML]{1E6DB2}} \color{white} .777 \,{\footnotesize KR} & {\cellcolor[HTML]{2E7EBC}} \color{white} .781 & {\cellcolor[HTML]{3C8CC3}} \color{white} .785 \,{\footnotesize GR} & {\cellcolor[HTML]{08306B}} \color{white} .760 \,{\footnotesize ID} & {\cellcolor[HTML]{1865AC}} \color{white} .775 & {\cellcolor[HTML]{3080BD}} \color{white} .782 \,{\footnotesize KR} & {\cellcolor[HTML]{F7FBFF}} \color{black} .832 & {\cellcolor[HTML]{77B5D9}} \color{white} .798 \,{\footnotesize ID} & {\cellcolor[HTML]{B4D3E9}} \color{black} .810 & {\cellcolor[HTML]{D4E4F4}} \color{black} .819 \,{\footnotesize DZ} \\
                & Low-Res & {\cellcolor[HTML]{083573}} \color{white} .650 & {\cellcolor[HTML]{82BBDB}} \color{white} .688 & {\cellcolor[HTML]{08306B}} \color{white} .649 \,{\footnotesize GR} & {\cellcolor[HTML]{083B7C}} \color{white} .652 & {\cellcolor[HTML]{08488E}} \color{white} .655 \,{\footnotesize MX} & {\cellcolor[HTML]{083979}} \color{white} .651 \,{\footnotesize CN} & {\cellcolor[HTML]{1E6DB2}} \color{white} .665 & {\cellcolor[HTML]{5FA6D1}} \color{white} .682 \,{\footnotesize US} & {\cellcolor[HTML]{B2D2E8}} \color{black} .698 & {\cellcolor[HTML]{B2D2E8}} \color{black} .698 \,{\footnotesize CN} & {\cellcolor[HTML]{E0ECF8}} \color{black} .712 & {\cellcolor[HTML]{F7FBFF}} \color{black} .720 \,{\footnotesize DZ} \\
                & Indo-Eu & {\cellcolor[HTML]{084082}} \color{white} .660 & {\cellcolor[HTML]{A6CEE4}} \color{white} .696 & {\cellcolor[HTML]{08306B}} \color{white} .656 \,{\footnotesize US} & {\cellcolor[HTML]{084B93}} \color{white} .662 & {\cellcolor[HTML]{0B559F}} \color{white} .665 \,{\footnotesize IR} & {\cellcolor[HTML]{084082}} \color{white} .660 \,{\footnotesize ID} & {\cellcolor[HTML]{2B7BBA}} \color{white} .674 & {\cellcolor[HTML]{58A1CF}} \color{white} .684 \,{\footnotesize IR} & {\cellcolor[HTML]{C4DAEE}} \color{black} .702 & {\cellcolor[HTML]{97C6DF}} \color{white} .694 \,{\footnotesize CN} & {\cellcolor[HTML]{D4E4F4}} \color{black} .707 & {\cellcolor[HTML]{F7FBFF}} \color{black} .718 \,{\footnotesize DZ} \\
                & Non-Indo & {\cellcolor[HTML]{083C7D}} \color{white} .704 & {\cellcolor[HTML]{5AA2CF}} \color{white} .737 & {\cellcolor[HTML]{08306B}} \color{white} .700 \,{\footnotesize CN} & {\cellcolor[HTML]{083C7D}} \color{white} .704 & {\cellcolor[HTML]{084B93}} \color{white} .708 \,{\footnotesize IR} & {\cellcolor[HTML]{084387}} \color{white} .706 \,{\footnotesize CN} & {\cellcolor[HTML]{1865AC}} \color{white} .717 & {\cellcolor[HTML]{3E8EC4}} \color{white} .730 \,{\footnotesize DZ} & {\cellcolor[HTML]{AACFE5}} \color{black} .754 & {\cellcolor[HTML]{AFD1E7}} \color{black} .755 \,{\footnotesize CN} & {\cellcolor[HTML]{CDE0F1}} \color{black} .764 & {\cellcolor[HTML]{F7FBFF}} \color{black} .781 \,{\footnotesize US} \\
                \midrule
                \multirow[c]{5}{*}{\makecell{Llama \\ 3.1-8B}} & All & {\cellcolor[HTML]{08458A}} \color{white} .707 & {\cellcolor[HTML]{EFF6FC}} \color{black} .755 & {\cellcolor[HTML]{083B7C}} \color{white} .705 \,{\footnotesize CN} & {\cellcolor[HTML]{084D96}} \color{white} .709 & {\cellcolor[HTML]{1966AD}} \color{white} .714 \,{\footnotesize US} & {\cellcolor[HTML]{08306B}} \color{white} .703 \,{\footnotesize ID} & {\cellcolor[HTML]{135FA7}} \color{white} .713 & {\cellcolor[HTML]{8CC0DD}} \color{white} .735 \,{\footnotesize US} & {\cellcolor[HTML]{C1D9ED}} \color{black} .743 & {\cellcolor[HTML]{4A98C9}} \color{white} .725 \,{\footnotesize ID} & {\cellcolor[HTML]{ADD0E6}} \color{black} .739 & {\cellcolor[HTML]{F7FBFF}} \color{black} .758 \,{\footnotesize US} \\
                & High-Res & {\cellcolor[HTML]{2B7BBA}} \color{white} .786 & {\cellcolor[HTML]{E7F0FA}} \color{black} .825 & {\cellcolor[HTML]{1B69AF}} \color{white} .782 \,{\footnotesize DZ} & {\cellcolor[HTML]{2C7CBA}} \color{white} .786 & {\cellcolor[HTML]{4090C5}} \color{white} .791 \,{\footnotesize US} & {\cellcolor[HTML]{08306B}} \color{white} .768 \,{\footnotesize ID} & {\cellcolor[HTML]{2B7BBA}} \color{white} .786 & {\cellcolor[HTML]{DCE9F6}} \color{black} .822 \,{\footnotesize US} & {\cellcolor[HTML]{A9CFE5}} \color{black} .809 & {\cellcolor[HTML]{0D57A1}} \color{white} .777 \,{\footnotesize ID} & {\cellcolor[HTML]{64A9D3}} \color{white} .798 & {\cellcolor[HTML]{F7FBFF}} \color{black} .830 \,{\footnotesize US} \\
                & Low-Res & {\cellcolor[HTML]{084488}} \color{white} .717 & {\cellcolor[HTML]{E6F0F9}} \color{black} .761 & {\cellcolor[HTML]{083776}} \color{white} .714 \,{\footnotesize CN} & {\cellcolor[HTML]{0D57A1}} \color{white} .720 & {\cellcolor[HTML]{1C6AB0}} \color{white} .724 \,{\footnotesize US} & {\cellcolor[HTML]{08306B}} \color{white} .713 \,{\footnotesize DZ} & {\cellcolor[HTML]{2272B6}} \color{white} .726 & {\cellcolor[HTML]{A4CCE3}} \color{white} .746 \,{\footnotesize US} & {\cellcolor[HTML]{C7DCEF}} \color{black} .752 & {\cellcolor[HTML]{60A7D2}} \color{white} .737 \,{\footnotesize ID} & {\cellcolor[HTML]{C3DAEE}} \color{black} .752 & {\cellcolor[HTML]{F7FBFF}} \color{black} .765 \,{\footnotesize IR} \\
                & Indo-Eu & {\cellcolor[HTML]{0A549E}} \color{white} .745 & {\cellcolor[HTML]{B5D4E9}} \color{black} .776 & {\cellcolor[HTML]{08519C}} \color{white} .744 \,{\footnotesize CN} & {\cellcolor[HTML]{1562A9}} \color{white} .748 & {\cellcolor[HTML]{2C7CBA}} \color{white} .753 \,{\footnotesize US} & {\cellcolor[HTML]{08306B}} \color{white} .737 \,{\footnotesize CN} & {\cellcolor[HTML]{3484BF}} \color{white} .755 & {\cellcolor[HTML]{BDD7EC}} \color{black} .777 \,{\footnotesize GR} & {\cellcolor[HTML]{C1D9ED}} \color{black} .777 & {\cellcolor[HTML]{125DA6}} \color{white} .747 \,{\footnotesize CN} & {\cellcolor[HTML]{A1CBE2}} \color{white} .772 & {\cellcolor[HTML]{F7FBFF}} \color{black} .792 \,{\footnotesize IR} \\
                & Non-Indo & {\cellcolor[HTML]{2A7AB9}} \color{white} .729 & {\cellcolor[HTML]{D9E7F5}} \color{black} .771 & {\cellcolor[HTML]{1967AD}} \color{white} .724 \,{\footnotesize DZ} & {\cellcolor[HTML]{2070B4}} \color{white} .726 & {\cellcolor[HTML]{2474B7}} \color{white} .728 \,{\footnotesize GR} & {\cellcolor[HTML]{08306B}} \color{white} .708 \,{\footnotesize DZ} & {\cellcolor[HTML]{1F6EB3}} \color{white} .726 & {\cellcolor[HTML]{8CC0DD}} \color{white} .751 \,{\footnotesize US} & {\cellcolor[HTML]{B4D3E9}} \color{black} .760 & {\cellcolor[HTML]{2474B7}} \color{white} .728 \,{\footnotesize KR} & {\cellcolor[HTML]{95C5DF}} \color{white} .753 & {\cellcolor[HTML]{F7FBFF}} \color{black} .783 \,{\footnotesize US} \\
                \midrule
                \multirow[c]{5}{*}{\makecell{Gemma \\ 2-27B}} & All & {\cellcolor[HTML]{08478D}} \color{white} .792 & {\cellcolor[HTML]{083B7C}} \color{white} .789 & {\cellcolor[HTML]{08306B}} \color{white} .786 \,{\footnotesize DZ} & {\cellcolor[HTML]{083B7C}} \color{white} .789 & {\cellcolor[HTML]{084E98}} \color{white} .793 \,{\footnotesize US} & {\cellcolor[HTML]{206FB4}} \color{white} .801 \,{\footnotesize MX} & {\cellcolor[HTML]{5BA3D0}} \color{white} .813 & {\cellcolor[HTML]{A0CBE2}} \color{white} .824 \,{\footnotesize KR} & {\cellcolor[HTML]{B2D2E8}} \color{black} .827 & {\cellcolor[HTML]{8CC0DD}} \color{white} .821 \,{\footnotesize MX} & {\cellcolor[HTML]{D8E7F5}} \color{black} .836 & {\cellcolor[HTML]{F7FBFF}} \color{black} .846 \,{\footnotesize IR} \\
                & High-Res & {\cellcolor[HTML]{125EA6}} \color{white} .851 & {\cellcolor[HTML]{2B7BBA}} \color{white} .856 & {\cellcolor[HTML]{08306B}} \color{white} .842 \,{\footnotesize DZ} & {\cellcolor[HTML]{08488E}} \color{white} .847 & {\cellcolor[HTML]{0B559F}} \color{white} .849 \,{\footnotesize IR} & {\cellcolor[HTML]{0F5AA3}} \color{white} .850 \,{\footnotesize MX} & {\cellcolor[HTML]{5FA6D1}} \color{white} .864 & {\cellcolor[HTML]{B5D4E9}} \color{black} .876 \,{\footnotesize IR} & {\cellcolor[HTML]{5AA2CF}} \color{white} .864 & {\cellcolor[HTML]{2B7BBA}} \color{white} .856 \,{\footnotesize MX} & {\cellcolor[HTML]{CDE0F1}} \color{black} .880 & {\cellcolor[HTML]{F7FBFF}} \color{black} .890 \,{\footnotesize DZ} \\
                & Low-Res & {\cellcolor[HTML]{083C7D}} \color{white} .799 & {\cellcolor[HTML]{083A7A}} \color{white} .798 & {\cellcolor[HTML]{08306B}} \color{white} .796 \,{\footnotesize IR} & {\cellcolor[HTML]{084488}} \color{white} .800 & {\cellcolor[HTML]{08519C}} \color{white} .804 \,{\footnotesize CN} & {\cellcolor[HTML]{0D57A1}} \color{white} .805 \,{\footnotesize ID} & {\cellcolor[HTML]{4695C8}} \color{white} .820 & {\cellcolor[HTML]{C6DBEF}} \color{black} .841 \,{\footnotesize KR} & {\cellcolor[HTML]{BAD6EB}} \color{black} .839 & {\cellcolor[HTML]{6DAFD7}} \color{white} .827 \,{\footnotesize ID} & {\cellcolor[HTML]{CEE0F2}} \color{black} .844 & {\cellcolor[HTML]{F7FBFF}} \color{black} .857 \,{\footnotesize KR} \\
                & Indo-Eu & {\cellcolor[HTML]{1865AC}} \color{white} .832 & {\cellcolor[HTML]{08306B}} \color{white} .821 & {\cellcolor[HTML]{0C56A0}} \color{white} .829 \,{\footnotesize GR} & {\cellcolor[HTML]{1561A9}} \color{white} .831 & {\cellcolor[HTML]{1F6EB3}} \color{white} .834 \,{\footnotesize IR} & {\cellcolor[HTML]{0E59A2}} \color{white} .829 \,{\footnotesize MX} & {\cellcolor[HTML]{60A7D2}} \color{white} .846 & {\cellcolor[HTML]{B5D4E9}} \color{black} .858 \,{\footnotesize GR} & {\cellcolor[HTML]{95C5DF}} \color{white} .853 & {\cellcolor[HTML]{A6CEE4}} \color{black} .855 \,{\footnotesize MX} & {\cellcolor[HTML]{D5E5F4}} \color{black} .865 & {\cellcolor[HTML]{F7FBFF}} \color{black} .874 \,{\footnotesize IR} \\
                & Non-Indo & {\cellcolor[HTML]{08306B}} \color{white} .796 & {\cellcolor[HTML]{0A539E}} \color{white} .803 & {\cellcolor[HTML]{084387}} \color{white} .800 \,{\footnotesize ID} & {\cellcolor[HTML]{09529D}} \color{white} .803 & {\cellcolor[HTML]{1966AD}} \color{white} .807 \,{\footnotesize US} & {\cellcolor[HTML]{115CA5}} \color{white} .805 \,{\footnotesize DZ} & {\cellcolor[HTML]{4D99CA}} \color{white} .818 & {\cellcolor[HTML]{CADDF0}} \color{black} .836 \,{\footnotesize KR} & {\cellcolor[HTML]{D6E5F4}} \color{black} .840 & {\cellcolor[HTML]{A1CBE2}} \color{white} .829 \,{\footnotesize MX} & {\cellcolor[HTML]{D6E6F4}} \color{black} .840 & {\cellcolor[HTML]{F7FBFF}} \color{black} .848 \,{\footnotesize GR} \\
                \bottomrule
            \end{tabular}}
        \caption{Mode Frequency  across models, language groups and mitigation strategies, in the format of \autoref{tab:main_results}: Steering, Persona Prompting and C-3PO Persona report the minimum, average and maximum over the eight personas (indicated by country code), and colour scales are applied row-wise, darker for lower values.}
        \label{tab:appendix_mode_frequency}
    \end{table*}
}

\newcommand{\tabAppendixErrorMetric}{
    \begin{table*}[htb]
        \centering
        \resizebox{\linewidth}{!}{
            \begin{tabular}{cccccccccccccc}
                \toprule
                \multirow{2}{*}{\textbf{Model}} & \multirow{2}{*}{\makecell{\textbf{Language}\\\textbf{Group}}} & \multirow{2}{*}{\textbf{Vanilla}} & \multirow{2}{*}{\makecell{\textbf{Few}\\\textbf{shot}}} & \multicolumn{3}{c}{\textbf{Steering}} & \multicolumn{3}{c}{\textbf{Persona}} & \multirow{2}{*}{\makecell{\textbf{C-3PO} \\ \textbf{Vanilla}}} & \multicolumn{3}{c}{\textbf{C-3PO Persona}} \\
                \cmidrule(lr){5-7} \cmidrule(lr){8-10} \cmidrule(lr){12-14}
                & & & & \textbf{Min} & \textbf{Avg} & \textbf{Max} & \textbf{Min} & \textbf{Avg} & \textbf{Max} & & \textbf{Min} & \textbf{Avg} & \textbf{Max} \\
                \midrule
                \multirow[c]{5}{*}{\makecell{Qwen \\ 2.5-3B}} & All & {\cellcolor[HTML]{0C56A0}} \color{white} .013 & {\cellcolor[HTML]{08306B}} \color{white} .003 & {\cellcolor[HTML]{1B69AF}} \color{white} .018 \,{\footnotesize CN} & {\cellcolor[HTML]{2979B9}} \color{white} .022 & {\cellcolor[HTML]{4D99CA}} \color{white} .031 \,{\footnotesize GR} & {\cellcolor[HTML]{3787C0}} \color{white} .026 \,{\footnotesize CN} & {\cellcolor[HTML]{84BCDB}} \color{white} .041 & {\cellcolor[HTML]{F7FBFF}} \color{black} .072 \,{\footnotesize US} & {\cellcolor[HTML]{2373B6}} \color{white} .020 & {\cellcolor[HTML]{083979}} \color{white} .005 \,{\footnotesize DZ} & {\cellcolor[HTML]{09529D}} \color{white} .012 & {\cellcolor[HTML]{4090C5}} \color{white} .028 \,{\footnotesize US} \\
                & High-Res & {\cellcolor[HTML]{084D96}} \color{white} .008 & {\cellcolor[HTML]{08306B}} \color{white} .003 & {\cellcolor[HTML]{105BA4}} \color{white} .010 \,{\footnotesize CN} & {\cellcolor[HTML]{2272B6}} \color{white} .014 & {\cellcolor[HTML]{58A1CF}} \color{white} .023 \,{\footnotesize GR} & {\cellcolor[HTML]{3787C0}} \color{white} .018 \,{\footnotesize ID} & {\cellcolor[HTML]{7CB7DA}} \color{white} .027 & {\cellcolor[HTML]{F7FBFF}} \color{black} .049 \,{\footnotesize US} & {\cellcolor[HTML]{4695C8}} \color{white} .020 & {\cellcolor[HTML]{083E81}} \color{white} .005 \,{\footnotesize DZ} & {\cellcolor[HTML]{1562A9}} \color{white} .012 & {\cellcolor[HTML]{82BBDB}} \color{white} .028 \,{\footnotesize US} \\
                & Low-Res & {\cellcolor[HTML]{084B93}} \color{white} .005 & {\cellcolor[HTML]{08306B}} \color{white} .000 & {\cellcolor[HTML]{084B93}} \color{white} .005 \,{\footnotesize IR} & {\cellcolor[HTML]{115CA5}} \color{white} .008 & {\cellcolor[HTML]{1966AD}} \color{white} .010 \,{\footnotesize US} & {\cellcolor[HTML]{083D7F}} \color{white} .003 \,{\footnotesize CN} & {\cellcolor[HTML]{56A0CE}} \color{white} .021 & {\cellcolor[HTML]{F7FBFF}} \color{black} .049 \,{\footnotesize US} & {\cellcolor[HTML]{08306B}} \color{white} .000 & {\cellcolor[HTML]{08306B}} \color{white} .000 \,{\footnotesize DZ} & {\cellcolor[HTML]{083573}} \color{white} .001 & {\cellcolor[HTML]{084B93}} \color{white} .005 \,{\footnotesize US} \\
                & Indo-Eu & {\cellcolor[HTML]{2676B8}} \color{white} .008 & {\cellcolor[HTML]{08306B}} \color{white} .000 & {\cellcolor[HTML]{08488E}} \color{white} .003 \,{\footnotesize MX} & {\cellcolor[HTML]{1C6BB0}} \color{white} .006 & {\cellcolor[HTML]{3F8FC5}} \color{white} .010 \,{\footnotesize US} & {\cellcolor[HTML]{2676B8}} \color{white} .008 \,{\footnotesize ID} & {\cellcolor[HTML]{6AAED6}} \color{white} .014 & {\cellcolor[HTML]{E5EFF9}} \color{black} .026 \,{\footnotesize US} & {\cellcolor[HTML]{5CA4D0}} \color{white} .013 & {\cellcolor[HTML]{135FA7}} \color{white} .005 \,{\footnotesize DZ} & {\cellcolor[HTML]{4594C7}} \color{white} .011 & {\cellcolor[HTML]{F7FBFF}} \color{black} .028 \,{\footnotesize US} \\
                & Non-Indo & {\cellcolor[HTML]{084F99}} \color{white} .010 & {\cellcolor[HTML]{08306B}} \color{white} .003 & {\cellcolor[HTML]{1764AB}} \color{white} .015 \,{\footnotesize US} & {\cellcolor[HTML]{206FB4}} \color{white} .018 & {\cellcolor[HTML]{3E8EC4}} \color{white} .026 \,{\footnotesize GR} & {\cellcolor[HTML]{2979B9}} \color{white} .020 \,{\footnotesize CN} & {\cellcolor[HTML]{7CB7DA}} \color{white} .037 & {\cellcolor[HTML]{F7FBFF}} \color{black} .066 \,{\footnotesize US} & {\cellcolor[HTML]{1764AB}} \color{white} .015 & {\cellcolor[HTML]{083A7A}} \color{white} .005 \,{\footnotesize DZ} & {\cellcolor[HTML]{08488E}} \color{white} .008 & {\cellcolor[HTML]{0E59A2}} \color{white} .013 \,{\footnotesize US} \\
                \midrule
                \multirow[c]{5}{*}{\makecell{Llama \\ 3.1-8B}} & All & {\cellcolor[HTML]{09529D}} \color{white} .028 & {\cellcolor[HTML]{08306B}} \color{white} .005 & {\cellcolor[HTML]{084E98}} \color{white} .026 \,{\footnotesize CN} & {\cellcolor[HTML]{0A549E}} \color{white} .030 & {\cellcolor[HTML]{0C56A0}} \color{white} .031 \,{\footnotesize DZ} & {\cellcolor[HTML]{1764AB}} \color{white} .041 \,{\footnotesize KR} & {\cellcolor[HTML]{4A98C9}} \color{white} .075 & {\cellcolor[HTML]{ABD0E6}} \color{black} .123 \,{\footnotesize DZ} & {\cellcolor[HTML]{A3CCE3}} \color{white} .118 & {\cellcolor[HTML]{85BCDC}} \color{white} .105 \,{\footnotesize DZ} & {\cellcolor[HTML]{D0E1F2}} \color{black} .146 & {\cellcolor[HTML]{F7FBFF}} \color{black} .182 \,{\footnotesize CN} \\
                & High-Res & {\cellcolor[HTML]{1764AB}} \color{white} .028 & {\cellcolor[HTML]{08306B}} \color{white} .005 & {\cellcolor[HTML]{125EA6}} \color{white} .026 \,{\footnotesize CN} & {\cellcolor[HTML]{1967AD}} \color{white} .030 & {\cellcolor[HTML]{1B69AF}} \color{white} .031 \,{\footnotesize DZ} & {\cellcolor[HTML]{083B7C}} \color{white} .010 \,{\footnotesize MX} & {\cellcolor[HTML]{125DA6}} \color{white} .025 & {\cellcolor[HTML]{4292C6}} \color{white} .049 \,{\footnotesize IR} & {\cellcolor[HTML]{8ABFDD}} \color{white} .072 & {\cellcolor[HTML]{4A98C9}} \color{white} .051 \,{\footnotesize DZ} & {\cellcolor[HTML]{BCD7EB}} \color{black} .088 & {\cellcolor[HTML]{F7FBFF}} \color{black} .120 \,{\footnotesize CN} \\
                & Low-Res & {\cellcolor[HTML]{08306B}} \color{white} .000 & {\cellcolor[HTML]{08306B}} \color{white} .000 & {\cellcolor[HTML]{08306B}} \color{white} .000 \,{\footnotesize DZ} & {\cellcolor[HTML]{08306B}} \color{white} .000 & {\cellcolor[HTML]{083573}} \color{white} .003 \,{\footnotesize US} & {\cellcolor[HTML]{084C95}} \color{white} .013 \,{\footnotesize US} & {\cellcolor[HTML]{64A9D3}} \color{white} .056 & {\cellcolor[HTML]{F7FBFF}} \color{black} .118 \,{\footnotesize DZ} & {\cellcolor[HTML]{7DB8DA}} \color{white} .064 & {\cellcolor[HTML]{A6CEE4}} \color{black} .077 \,{\footnotesize DZ} & {\cellcolor[HTML]{CCDFF1}} \color{black} .092 & {\cellcolor[HTML]{EAF3FB}} \color{black} .110 \,{\footnotesize CN} \\
                & Indo-Eu & {\cellcolor[HTML]{0E59A2}} \color{white} .026 & {\cellcolor[HTML]{08306B}} \color{white} .005 & {\cellcolor[HTML]{084F99}} \color{white} .020 \,{\footnotesize ID} & {\cellcolor[HTML]{0A549E}} \color{white} .023 & {\cellcolor[HTML]{0E59A2}} \color{white} .026 \,{\footnotesize IR} & {\cellcolor[HTML]{084F99}} \color{white} .020 \,{\footnotesize KR} & {\cellcolor[HTML]{1D6CB1}} \color{white} .035 & {\cellcolor[HTML]{57A0CE}} \color{white} .061 \,{\footnotesize IR} & {\cellcolor[HTML]{A9CFE5}} \color{black} .090 & {\cellcolor[HTML]{84BCDB}} \color{white} .077 \,{\footnotesize DZ} & {\cellcolor[HTML]{C9DDF0}} \color{black} .103 & {\cellcolor[HTML]{F7FBFF}} \color{black} .133 \,{\footnotesize CN} \\
                & Non-Indo & {\cellcolor[HTML]{083674}} \color{white} .003 & {\cellcolor[HTML]{08306B}} \color{white} .000 & {\cellcolor[HTML]{083674}} \color{white} .003 \,{\footnotesize CN} & {\cellcolor[HTML]{084387}} \color{white} .007 & {\cellcolor[HTML]{084B93}} \color{white} .010 \,{\footnotesize US} & {\cellcolor[HTML]{08519C}} \color{white} .013 \,{\footnotesize ID} & {\cellcolor[HTML]{66ABD4}} \color{white} .049 & {\cellcolor[HTML]{F7FBFF}} \color{black} .100 \,{\footnotesize DZ} & {\cellcolor[HTML]{66ABD4}} \color{white} .049 & {\cellcolor[HTML]{7AB6D9}} \color{white} .054 \,{\footnotesize DZ} & {\cellcolor[HTML]{CBDEF1}} \color{black} .077 & {\cellcolor[HTML]{F2F8FD}} \color{black} .097 \,{\footnotesize ID} \\
                \midrule
                \multirow[c]{5}{*}{\makecell{Gemma \\ 2-27B}} & All & {\cellcolor[HTML]{0E58A2}} \color{white} .010 & {\cellcolor[HTML]{08306B}} \color{white} .005 & {\cellcolor[HTML]{08306B}} \color{white} .005 \,{\footnotesize DZ} & {\cellcolor[HTML]{084285}} \color{white} .007 & {\cellcolor[HTML]{084488}} \color{white} .008 \,{\footnotesize US} & {\cellcolor[HTML]{084488}} \color{white} .008 \,{\footnotesize DZ} & {\cellcolor[HTML]{1967AD}} \color{white} .012 & {\cellcolor[HTML]{4594C7}} \color{white} .018 \,{\footnotesize US} & {\cellcolor[HTML]{CADDF0}} \color{black} .031 & {\cellcolor[HTML]{3080BD}} \color{white} .015 \,{\footnotesize DZ} & {\cellcolor[HTML]{A4CCE3}} \color{white} .027 & {\cellcolor[HTML]{F7FBFF}} \color{black} .038 \,{\footnotesize US} \\
                & High-Res & {\cellcolor[HTML]{3080BD}} \color{white} .010 & {\cellcolor[HTML]{0E58A2}} \color{white} .005 & {\cellcolor[HTML]{0E58A2}} \color{white} .005 \,{\footnotesize DZ} & {\cellcolor[HTML]{1B69AF}} \color{white} .007 & {\cellcolor[HTML]{1D6CB1}} \color{white} .008 \,{\footnotesize US} & {\cellcolor[HTML]{08306B}} \color{white} .000 \,{\footnotesize DZ} & {\cellcolor[HTML]{2979B9}} \color{white} .009 & {\cellcolor[HTML]{7AB6D9}} \color{white} .018 \,{\footnotesize US} & {\cellcolor[HTML]{9AC8E0}} \color{white} .020 & {\cellcolor[HTML]{3080BD}} \color{white} .010 \,{\footnotesize DZ} & {\cellcolor[HTML]{97C6DF}} \color{white} .020 & {\cellcolor[HTML]{F7FBFF}} \color{black} .033 \,{\footnotesize US} \\
                & Low-Res & {\cellcolor[HTML]{3787C0}} \color{white} .008 & {\cellcolor[HTML]{08306B}} \color{white} .000 & {\cellcolor[HTML]{08306B}} \color{white} .000 \,{\footnotesize DZ} & {\cellcolor[HTML]{1B69AF}} \color{white} .005 & {\cellcolor[HTML]{3787C0}} \color{white} .008 \,{\footnotesize GR} & {\cellcolor[HTML]{3787C0}} \color{white} .008 \,{\footnotesize DZ} & {\cellcolor[HTML]{4F9BCB}} \color{white} .010 & {\cellcolor[HTML]{58A1CF}} \color{white} .010 \,{\footnotesize US} & {\cellcolor[HTML]{F7FBFF}} \color{black} .023 & {\cellcolor[HTML]{3787C0}} \color{white} .008 \,{\footnotesize DZ} & {\cellcolor[HTML]{9DCAE1}} \color{white} .014 & {\cellcolor[HTML]{F7FBFF}} \color{black} .023 \,{\footnotesize GR} \\
                & Indo-Eu & {\cellcolor[HTML]{2070B4}} \color{white} .010 & {\cellcolor[HTML]{08468B}} \color{white} .005 & {\cellcolor[HTML]{08468B}} \color{white} .005 \,{\footnotesize DZ} & {\cellcolor[HTML]{0E58A2}} \color{white} .007 & {\cellcolor[HTML]{105BA4}} \color{white} .008 \,{\footnotesize US} & {\cellcolor[HTML]{105BA4}} \color{white} .008 \,{\footnotesize DZ} & {\cellcolor[HTML]{2F7FBC}} \color{white} .012 & {\cellcolor[HTML]{6CAED6}} \color{white} .018 \,{\footnotesize US} & {\cellcolor[HTML]{8DC1DD}} \color{white} .020 & {\cellcolor[HTML]{08306B}} \color{white} .003 \,{\footnotesize DZ} & {\cellcolor[HTML]{64A9D3}} \color{white} .017 & {\cellcolor[HTML]{F7FBFF}} \color{black} .033 \,{\footnotesize US} \\
                & Non-Indo & {\cellcolor[HTML]{3787C0}} \color{white} .008 & {\cellcolor[HTML]{08306B}} \color{white} .000 & {\cellcolor[HTML]{08306B}} \color{white} .000 \,{\footnotesize DZ} & {\cellcolor[HTML]{135FA7}} \color{white} .004 & {\cellcolor[HTML]{1B69AF}} \color{white} .005 \,{\footnotesize US} & {\cellcolor[HTML]{1B69AF}} \color{white} .005 \,{\footnotesize DZ} & {\cellcolor[HTML]{3383BE}} \color{white} .007 & {\cellcolor[HTML]{58A1CF}} \color{white} .010 \,{\footnotesize US} & {\cellcolor[HTML]{E1EDF8}} \color{black} .020 & {\cellcolor[HTML]{58A1CF}} \color{white} .010 \,{\footnotesize US} & {\cellcolor[HTML]{B4D3E9}} \color{black} .016 & {\cellcolor[HTML]{F7FBFF}} \color{black} .023 \,{\footnotesize GR} \\
                \bottomrule
            \end{tabular}}
        \caption{Error Metric ($\downarrow$) across models, language groups and mitigation strategies, in the format of \autoref{tab:main_results}: Steering, Persona Prompting and C-3PO Persona report the minimum, average and maximum over the eight personas (indicated by country code), and colour scales are applied row-wise, darker for lower values.}
        \label{tab:appendix_error_metric}
    \end{table*}
}

\title{Mitigating Cross-Lingual Cultural Inconsistencies in LLMs via\\Consensus-Driven Preference Optimisation}

\author{%
    Lucas Resck\textnormal{\textsuperscript{1}} \and%
    Isabelle Augenstein\textnormal{\textsuperscript{2}} \and%
    Anna Korhonen\textnormal{\textsuperscript{1}} \\%
    \textsuperscript{1}Language Technology Lab, University of Cambridge \\%
    \textnormal{\textsuperscript{2}}University of Copenhagen \\%
    \small{%
        \texttt{\{ler44, alk23\}@cam.ac.uk}, \texttt{augenstein@di.ku.dk}%
    }%
}

\renewcommand{\subsectionautorefname}{Section}
\renewcommand{\subsubsectionautorefname}{Section}

\makeatletter
\def\paragraph{\@startsection{paragraph}{4}{\z@}%
    {0.75ex plus 0.25ex minus 0.1ex}{-1em}{\normalsize\bfseries}}
\makeatother

\begin{document}

    \maketitle
    \begin{abstract}
        Multilingual large language models (MLLMs) often answer culturally ambiguous questions inconsistently across languages, even for a fixed user persona -- e.g., Shakespeare in English but Cervantes in Spanish for the same British user. To robustly quantify this Cross-lingual Cultural Inconsistency, we introduce \textit{Singleton Fleiss's $\kappa_S$}, a metric mathematically resilient to hallucinations. For mitigation, we propose \textit{Cross-lingual Cultural Consistent Preference Optimisation} (C-3PO), a consensus-driven alignment framework.
C-3PO raises $\kappa_S$ by 0.07--0.10 over unaligned models, above every prompting and representation steering baseline averaged across personas, whilst largely preserving explicit user identities, cultural neutrality and intrinsic cultural knowledge; where it does not, a soft-consensus variant recovers most of the cost.
Although trained only on multiple-choice questions, C-3PO also improves consistency in open-ended generation.
Empirical evaluations demonstrate this inconsistency is more pronounced for lower-resource languages like Indonesian and Persian.
Finally, early decoding of intermediate layers reveals that MLLMs implicitly personalise outputs towards the prompt language's stereotypical culture as forward-pass representations stabilise.%
\footnote{Code and data will be released upon publication.}

    \end{abstract}

    \section{Introduction}

    Despite their impressive capabilities \cite{wu_enhancing_2025, cui_multilingual_2025, jiang_few-shot_2025}, multilingual large language models (MLLMs) frequently exhibit inconsistent behaviour when the prompt's language changes \cite{lu_cultural_2025, fierro_factual_2022, qi_cross-lingual_2023, alkhamissi_investigating_2024}. While such adaptation is often desirable for locale-specific queries \cite{veselovsky_localized_2025}, it becomes a critical failure when a user's identity is explicitly defined \cite{bulte_llms_2025}. For instance, given a fixed British persona and an ambiguous everyday knowledge query about literature, the prompt's language frequently overwrites the system persona -- yielding Shakespeare in English but Cervantes in Spanish (\autoref{fig:cross_lingual_cultural_inconsistency}).

    \figCrossLingualCulturalInconsistency

    {
        Therefore, cross-lingual inconsistency in this constrained context represents a clear failure of instruction following, stemming from the tight entanglement between language and culture \cite{yu_entangled_2026}.
    }%
    Consequently, users may receive conflicting information simply by querying in a different language, exacerbating perceptions of cultural bias \cite{zhou_does_2025}.

    While various methods exist to measure and mitigate general LLM inconsistency \cite{ifergan_beneath_2025, agarwal_aligning_2025, bu_alignx_2025}, the specific phenomenon of cross-lingual cultural inconsistency in MLLMs remains largely underexplored.
    {%
        We focus on the specific undesirable behaviour where models fail to leverage the user's explicit identity (e.g., via system prompts), providing instead culturally personalised answers dictated solely by the prompt's language.
        Consistency and persona adherence are therefore orthogonal requirements -- a model can answer identically in every language while ignoring the persona, or follow the persona in some languages only -- and we measure both.
    }

    In this work, we investigate and formalise the problem of Cross-lingual Cultural Inconsistency (CCI) in MLLMs.
    Specifically, we make the following contributions:
    \begin{itemize}[topsep=2pt, partopsep=0.5pt, itemsep=1pt, parsep=1pt]
        \item Drawing inspiration from traditional inter-annotator agreement metrics, we introduce \textit{Singleton Fleiss's $\kappa_S$}, a novel metric mathematically proven to evaluate cross-lingual agreement robustly, even in the presence of invalid or hallucinated responses.
        \item We propose \textit{Cross-lingual Cultural Consistent Preference Optimisation} (C-3PO), a self-supervised mitigation framework that leverages consensus among multilingual responses to align the model's representations and improve cultural consistency beyond the average of every baseline across three model families, while largely preserving explicit user identities, cultural neutrality and intrinsic cultural knowledge.
        \item We show that C-3PO's gains transfer to open-ended generation and ablate its consensus rule; a soft-consensus variant recovers most of C-3PO's cost where it has one.
        \item We empirically show that CCI is associated with language resource levels, with lower-resource languages showing more severe inconsistency.
        \item We conduct an interpretability analysis and provide layer-wise evidence, in all three models, that models implicitly personalise their answers towards the prompt language's stereotypical culture as their forward-pass representations stabilise.
    \end{itemize}

\section{Related Work}
    \label{sec:related_work}

    {
        \paragraph{Measuring and Mitigating Inconsistency.}

            \leavevmode Multilingual inconsistency has been measured through contrasts between explicit cultural cues and implicit language signals \cite{veselovsky_localized_2025}, intra-language self-consistency \cite{fierro_factual_2022} and, closest to our methodology, identical answers to identical cross-lingual queries \cite{qi_cross-lingual_2023, ifergan_beneath_2025}. Mitigations include multilingual parallel batches \cite{agarwal_aligning_2025}, contrastive alignment of internal representations \cite{bu_alignx_2025} and layer shortcuts for factual consistency \cite{wang_lost_2025}.
    }

    \paragraph{Consensus as Supervision.}
        A model's own agreement is an established training signal, from majority votes over sampled reasoning paths (LMSI; \citealp{huang_large_2023}) to Semantic Voting on open-ended tasks \cite{jiang_semantic_2026}. Across languages, CALM \cite{wang_calm_2025} prefers the most cross-lingually self-consistent answer for DPO, while MAPO \cite{she_mapo_2024} and CM-Align \cite{zhang_cm-align_2025} align other languages to a dominant language or an English reference, and multilingual preference data benefits preference optimisation across languages \cite{dang_rlhf_2024}. These methods are monolingual or target questions with a correct answer -- CALM, for instance, scopes consistency to culture-independent questions -- and most cross-lingual ones anchor the other languages to English or a dominant language (\autoref{tab:related_work}).
        \tabRelatedWork

    \paragraph{Our Positioning.}

        In contrast to prior work evaluating factual accuracy or performance gaps, we explicitly isolate \textit{Cross-lingual Cultural Inconsistency} -- the phenomenon where a model generates semantically divergent responses to a cultural query based solely on the prompt's language. Under an explicit persona, this divergence is a failure of instruction following; we hypothesise that it is driven by implicit cultural personalisation \cite{neplenbroek_reading_2025} and mitigate it with consensus-driven preference optimisation. Unlike the methods in \autoref{tab:related_work}, C-3PO operates where no answer is correct, forms its target symmetrically from all languages without a pivot language or external model, and balances which languages it corrects.

\section{Cross-lingual Cultural Inconsistency}
    \label{sec:cross_lingual_cultural_inconsistency}

    \subsection{Definition and Formalisation}
        \label{sec:definition_and_formalisation}

        Consider a scenario where a user with a fixed identity (persona) interacts with LLMs across multiple languages. Ideally, an aligned model should maintain semantic consistency regarding the user's persona, regardless of the input language. However, as illustrated in \autoref{fig:cross_lingual_cultural_inconsistency}, LLMs frequently exhibit divergent behaviours based on the prompt language, especially for topics like everyday knowledge. We define \textit{Cross-lingual Cultural Inconsistency} (CCI) as the divergence of model outputs across different languages given a fixed user persona.

        Formally, we assume the model $M$ receives a user persona $u$ (e.g., via a system prompt, user profile or interaction history) and an input query $x$ formulated in a language $l$. Let $y = M(u, x, l)$ denote the generated response. The model is considered \textit{inconsistent} if there exist two languages $l_1, l_2$ such that, for the same user persona $u$ and query content $x$,
        \begin{equation*}
            M(u, x, l_1) \neq M(u, x, l_2).
        \end{equation*}
        This inequality implies semantic disagreement rather than lexical difference. It suggests that the prompt's language $l$ acts as a confounding variable, implicitly overriding the explicit persona $u$ with cultural priors associated with $l$ (e.g., a Spanish prompt triggering Spanish cultural associations in \autoref{fig:cross_lingual_cultural_inconsistency}). We frame this as a failure of instruction following, where implicit language-driven personalisation outweighs explicit user-driven constraints.

    \subsection{Consistency vs. Performance}
        \label{sec:consistency_vs_performance}

        It is crucial to distinguish cross-lingual \textit{consistency} from cross-lingual \textit{performance} (or accuracy). High performance in multiple languages does not guarantee consistency between them. For instance, a model could achieve identical accuracy scores in two languages (e.g., 60\%) yet fail on completely disjoint subsets of questions; in such a case, performance is stable, but cross-lingual consistency is low. Conversely, if a model hallucinates the exact same incorrect answer in both languages, performance is zero, but consistency is maximised. In the domain of cultural knowledge, particularly everyday knowledge, where ground truth is often subjective or undefined, we cannot rely solely on accuracy. Here, consistency serves as a vital proxy for model robustness; an aligned model should not alter its stance on subjective cultural topics merely because the conversation language has changed, assuming the user persona remains constant.

        \paragraph{Consistency vs. Persona Adherence.}
            Consistency is equally distinct from persona adherence, the extent to which answers follow the stated persona. The two are \textbf{orthogonal}: a model that gives the US-associated answer in every language to a user from Mexico is perfectly consistent yet ignores the persona, whereas a model that follows the persona only when prompted in Spanish adheres in one language and not in the others. We therefore measure both: $\kappa_S$ quantifies consistency across languages (\autoref{sec:inconsistency_measurement}) and persona adherence is evaluated separately (\autoref{sec:cultural_and_persona_alignment}).

\section{Methodology}

    \subsection{Dataset Construction}
        \label{sec:dataset_construction}

        \figDatasetConstruction

        To systematically evaluate CCI, we construct a multilingual parallel dataset of everyday knowledge queries derived from the BLEnD benchmark \cite{myung_blend_2024} (\autoref{fig:dataset_construction}). We select its multiple-choice question (MCQ) subset (\autoref{fig:blend_mcq_format}), as this discrete format facilitates evaluation via standard agreement metrics like Fleiss's $\kappa$. Because the original MCQs are English-only, we extract parallel translated questions from BLEnD's short-answer (SAQ) subset across eight diverse languages: English, Spanish, Chinese, Arabic, Indonesian, Korean, Greek and Persian.
        \autoref{fig:dataset_construction} summarises the construction process; \autoref{appendix:dataset_construction_details} details the country-reference neutralisation, question filtering and option translation, and \autoref{appendix:dataset_quality_assurance} validates the dataset's integrity.

        We split the base questions 70/10/20 into training, validation and test sets, using the validation set for hyperparameter tuning and the test set -- 391 multiple-choice items from 79 base questions, i.e. 3,128 prompts per method -- for evaluation. For analysis, we group the languages by resource level, splitting them at a 1\% share of Common Crawl \cite{resck_explainability_2025} into Higher-Resource (English, Chinese, Spanish) and Lower-Resource (Indonesian, Persian, Korean, Arabic, Greek),\footnote{We utilise the \texttt{CC-MAIN-2026-08} crawl. Indonesian (1.01\%) is assigned to the lower-resource group as its frequency is closer to Persian's (0.88\%) than Spanish's (4.47\%).} and by family into Indo-European (English, Spanish, Persian, Greek) and Non-Indo-European (Chinese, Indonesian, Korean, Arabic).

    \subsection{Inconsistency Measurement}
        \label{sec:inconsistency_measurement}

        To quantify cross-lingual inconsistency independently of ground-truth accuracy, we build upon Fleiss's $\kappa$ \cite{fleiss_measuring_1971} to propose our novel \textit{Singleton Fleiss's $\kappa_S$} metric.
        The original $\kappa$ is a standard statistical measure of inter-rater agreement for categorical data increasingly adopted for multilingual evaluation \cite{zaghouani_emohopespeech_2025, riabi_beyond_2025}. Unlike exact-match metrics -- which misleadingly report 100\% consistency if a biased model pathologically predicts a single option -- it robustly accounts for chance agreement using marginal distributions.

        While standard $\kappa$ captures valid cross-lingual agreement, it assumes all responses map cleanly to a predefined valid set ($\mathcal{V}$). In practice, LLMs frequently hallucinate or generate invalid formats. Discarding these errors induces survivorship bias, artificially inflating consistency scores. To address this, we introduce \textit{Singleton Fleiss's $\kappa_S$}.

        \begin{definition}[Singleton Fleiss's $\kappa_S$]
            \label{def:singleton_fleiss_kappa}
            We extend the valid answer set $\mathcal{V}$ with a set of dynamically generated singleton (invalid) answers $\mathcal{U}$, yielding $\mathcal{V}' = \mathcal{V} \cup \mathcal{U}$. Every invalid response maps to a strictly unique element in $\mathcal{U}$. For $N$ samples and $n$ languages, let $n_{ij}$ denote the number of languages assigning category $j \in \mathcal{V}'$ to sample $i$.
            We then compute observed agreement $P_o$, expected agreement $P_e$ and $\kappa_S$ as follows:
            \begin{align*}
                P_o = \frac{1}{Nn(n-1)} \sum_{i=1}^N \sum_{j \in \mathcal{V}'} n_{ij}(n_{ij} - 1), \\
                P_e = \sum_{j \in \mathcal{V}'} \left( \frac{1}{Nn} \sum_{i=1}^N n_{ij} \right)^2, \ \
                \kappa_S = \frac{P_o - P_e}{1 - P_e}.
            \end{align*}
            When every answer is the same option, $P_e = 1$ and we set $\kappa_S = 0$.
        \end{definition}

        By treating errors as unique singletons, $\kappa_S$ mathematically penalises inconsistency without requiring ad-hoc sample exclusion. Crucially, we prove that $\kappa_S$ asymptotically converges to the standard valid-category $\kappa$ as dataset size increases, with a gap of at most $O(\epsilon/Nn)$ for a proportion $\epsilon$ of invalid responses (\autoref{appendix:proof_singleton_kappa}).
        Alongside $\kappa_S$, we track three complementary metrics: \textit{Soft Consistency} (average pairwise agreement), \textit{Hard Consistency} (proportion of unanimous agreement) and \textit{Mode Frequency} (dominant answer selection rate); mathematical definitions in \autoref{appendix:metric_definitions}.

    \subsection{Cross-lingual Cultural Consistent Preference Optimisation (C-3PO)}
        \label{sec:c_3po}

        \figCThreePO

        To systematically mitigate cross-lingual inconsistency, we propose \textit{Cross-lingual Cultural Consistent Preference Optimisation} (C-3PO), a novel self-supervised framework. C-3PO leverages the parallel structure of our dataset and the model's own multilingual generations to establish a cross-lingual consensus, which is subsequently used to align responses across languages.

        As illustrated in \autoref{fig:c_3po}, the pipeline operates in three phases. First, given a culturally neutral query (no country-specific references), we prompt the base LLM to generate responses across all $n=8$ languages in the training set. Second, we extract a cross-lingual consensus for each query separately: the plurality answer, i.e. the option chosen by the most languages, provided that at least two languages choose it and no other option ties with it. Invalid answers are excluded from this count.
        For samples exhibiting a valid consensus, we construct language-specific preference pairs $(y_w, y_l)$. The consensus answer is universally assigned as the chosen response $y_w$. The rejected response $y_l$ is defined conditionally: for languages where the model initially disagreed with the consensus, $y_l$ is the actual divergent output; for languages that already agreed or whose answer was invalid, $y_l$ is uniformly sampled from the remaining non-consensus options.

        To prevent the optimisation process from disproportionately biasing the model towards languages that frequently dictate the consensus (for instance, higher-resource languages), we apply an undersampling strategy to balance the representation of consensual and divergent languages.
        Concretely, a query in which some languages diverge is kept with a probability set by the least frequently divergent of them, so that languages that rarely disagree with the consensus are not outnumbered by those that often do (retained pairs per language in \autoref{appendix:hyperparameters}).
        During the final fine-tuning phase, data is structured into parallel batches -- each containing the exact same query translated across all $n$ languages. We then optimise the model using Direct Preference Optimisation (DPO)~\cite{rafailov_direct_2023}. This parallel batching ensures that gradient updates simultaneously pull disparate language representations towards a unified semantic anchor.

        By mining preferences directly from the model's internal consensus in a language-balanced manner, C-3PO circumvents the need for costly human annotations or culturally subjective ground truths. Furthermore, integrating DPO with Low-Rank Adaptation (LoRA)~\cite{hu_lora_2022} provides a highly efficient and scalable solution that avoids the brittleness of heuristic baselines such as ad-hoc persona and few-shot prompting or vector steering.

    \subsection{Baseline Mitigation Strategies}
        \label{sec:mitigation_strategies}

        To benchmark C-3PO's efficacy, we implement three established behavioural steering baselines, alongside the unaltered model (Vanilla). Persona Prompting and Persona Vector Steering are run once per persona, for each of the eight nationalities, so we report their minimum, average and maximum over personas.

        \paragraph{Persona Prompting.}

            {
                We explicitly condition the model to adopt a specific nationality (e.g., \textit{``You are a person from Mexico (...)''};} see \autoref{fig:persona_prompt}), mirroring the formalisation in \autoref{sec:definition_and_formalisation}. This aligns with established findings that LLMs adapt their outputs given contextual background information \cite{veselovsky_localized_2025, ying_disentangling_2025}.

        \paragraph{Few-shot Prompting.}

            As an in-context learning analogue to C-3PO \cite{mosbach_few-shot_2023}, we prepend demonstrations of identical cross-lingual queries that consistently yield the same culturally appropriate answer. To prevent degenerate mode-seeking behaviour (i.e., the model blindly repeating a single answer), we ensure the few-shot examples encompass diverse target labels.

        \paragraph{Persona Vector Steering.}

            Following recent representation engineering methods \cite{veselovsky_localized_2025, ghandeharioun_whos_2024}, we steer the model's latent space using persona-specific intervention vectors. These vectors are computed as the mean difference in residual stream activations (at the final token position) between prompt pairs with and without the persona instructions.
            This way, we can ``simulate'' the persona prompt but control the steering intensity more precisely.
            We extract language-specific vectors and sweep over layer combinations to identify the optimal configuration.

\section{Experiments on Consistency}

    \subsection{Experimental Setup}
        \label{sec:experimental_setup}

        \paragraph{Models.}

            We evaluate three open-weight multilingual LLMs spanning diverse architectures and parameter scales: \texttt{Gemma-2-27b-it}~\cite{team_gemma_2024}, \texttt{Llama-3.1-8B-Instruct}~\cite{grattafiori_llama_2024} and \texttt{Qwen2.5-3B-Instruct}~\cite{qwen_qwen25_2025}.

        \paragraph{Metrics.}

            We employ Singleton Fleiss's $\kappa_S$ as our primary consistency metric, supplemented by its bootstrap variance ($\sigma^2_{\kappa_S}$, 1,000 resamples of base questions, each carrying all of its multiple-choice items).
            We also report Soft and Hard Consistency scores, Mode Frequency and Model Error for a comprehensive evaluation.
            Metrics are computed for each of the language groups defined in \autoref{sec:dataset_construction}.

        \paragraph{Hyperparameters.}

            Comprehensive hyperparameter settings for all methods are optimised on the validation set and detailed in \autoref{appendix:hyperparameters}, together with the selected configurations, run counts, seeds and the retained training data.

    \subsection{Results on Consistency}
        \label{sec:results_on_consistency}

        \tabMainResults

        \autoref{tab:main_results} reports the Singleton Fleiss's $\kappa_S$ consistency scores across all models and mitigation strategies on the test set.
        We report the performance of the unmitigated (Vanilla) models alongside our baselines and C-3PO.
        To ensure our evaluation strictly aligns with the formalisation of CCI (which assumes a fixed user persona), we evaluate our framework both in its unprompted state (C-3PO Vanilla) and combined with Persona Prompting (C-3PO Persona).
        For all persona-based methods, we detail the minimum, average and maximum across the eight personas.

        {
            Unmitigated models naturally exhibit higher consistency on higher-resource languages, and overall consistency scales predictably with model capacity. Crucially, \textbf{both C-3PO variants outperform the average of every baseline mitigation strategy across all models and language groups}, with C-3PO Vanilla raising $\kappa_S$ over Vanilla by 0.07--0.10 on all eight languages. Only the best single persona of Persona Prompting -- a choice that requires knowing the test results -- exceeds C-3PO Vanilla, in four of the fifteen model and language-group cells (the Higher-Resource and Indo-European groups of Llama and Gemma). Comparing like with like, the best persona of C-3PO Persona exceeds the best persona of every baseline in all fifteen cells.
        }

        Comparing baselines, Few-shot Prompting proves more effective for smaller models, whereas Persona Prompting excels in larger models, likely due to their advanced instruction-following capabilities. These primary findings hold across alternative consistency metrics (Tables \ref{tab:appendix_soft_consistency}--\ref{tab:appendix_error_metric}). Notably, while Few-shot artificially inflates Soft and Hard Consistency by collapsing to a single mode (i.e., indiscriminately repeating an answer, as evidenced by \autoref{tab:appendix_mode_frequency}), $\kappa_S$ penalises this degenerate behaviour.

    \subsection{Consistency vs. Language Resource Level}
        \label{sec:consistency_vs_resource_level}

        \figConsistencyVariationQwen

        \autoref{tab:main_results}
        highlights that the Lower-Resource group suffers from markedly lower $\kappa_S$ scores across all models and nearly all mitigation strategies.
        To empirically isolate this effect, we measure consistency variation as languages are incrementally added to the evaluation pool based on their resource availability, ranked by their Common Crawl distribution~\cite{resck_explainability_2025,lai_okapi_2023}. We simulate two scenarios: incrementally adding languages in decreasing order of resourcefulness ({\color{NavyBlue} Higher-to-Lower}), and in increasing order ({\color{teal} Lower-to-Higher}).

        As illustrated for Qwen-2.5-3B in \autoref{fig:consistency_variation_qwen}, incorporating languages in decreasing order of resourcefulness degrades consistency across all methods, whereas the reverse order improves it. Because the sets of evaluated languages are identical at the middle, these opposing trajectories point to \textbf{resource availability as a major factor in consistency degradation}.
        Resource level is, however, confounded with other properties of the languages -- script, family, region, translation quality and each language's share of the pretraining data -- which eight languages cannot disentangle, so we read this ordering as an association rather than a causal effect.

        Llama-3.1-8B and Gemma-2-27B show a similar but weaker trend (\autoref{fig:consistency_variation_plots}). Notably, C-3PO remains the most consistent method for every model throughout the addition of languages.

    \subsection{Open-Ended Transfer and Ablation}
        \label{sec:further_analyses_summary}

        \textbf{Although trained only on multiple-choice questions, C-3PO also improves consistency when models answer freely}: asked the test questions without their options, with answers mapped to categories by rule-based lexicon matching, it raises $\kappa_S$ in 29 of 30 comparisons across models, language groups and method pairs, by relative margins similar to those with multiple choice. An ablation of its consensus rule shows that neither a strict majority nor a soft consensus improves on its plurality vote, although the full soft-consensus variant recovers most of the persona adherence and cultural knowledge that C-3PO costs Llama-3.1-8B (\autoref{sec:cultural_and_persona_alignment}, \autoref{sec:further_analyses}).

\section{Cultural and Persona Alignment}
    \label{sec:cultural_and_persona_alignment}

    \tabCThreePOImplications

    To ensure that C-3PO's consistency gains do not come at the cost of other alignment dimensions, we examine its impact on (1) systemic cultural bias, (2) explicit persona adherence and (3) intrinsic cultural knowledge (cultural erasure).

    \paragraph{1) Cultural Bias.}

        To investigate whether our consensus-driven method inadvertently privileges specific (e.g., high-resource) cultures, we map the models' answers back to the original BLEnD answer-to-country annotations.
        \autoref{tab:c3po_implications} (i) reports the country selection rate for each model and non-persona method in the test set, averaged across countries (detailed breakdown in \autoref{tab:cultural_bias_expanded}). All methods, including C-3PO Vanilla, exhibit highly similar country selection rates. C-3PO keeps their spread within 0.2 points of Vanilla's standard deviation for every model and moves no country's rate by more than 1.2 points, which indicates that C-3PO fine-tuning does not induce systemic cultural bias.
        The consensus answers C-3PO is trained to prefer are somewhat more concentrated, most visibly on the US, whose share rises by up to 4.1 points over the unaligned model's (Preference rows in \autoref{tab:cultural_bias_expanded}), but this concentration does not carry over to the fine-tuned models.

    \paragraph{2) Persona Adherence.}

        To verify that the model correctly applies explicitly assigned user identities (e.g., a Mexican persona yielding answers aligned with Mexico), we evaluate persona-country match accuracy. \autoref{tab:c3po_implications} (ii) presents this accuracy for each persona-based method.
        Crucially, most models maintain comparable accuracy between the baseline Persona method and C-3PO Persona, fluctuating by merely $(-2.5\%, +0.8\%)$. This confirms that C-3PO's consensus mechanism does not disregard explicit user personas. The sole exception is Llama, whose adherence drops from 49.9 to 43.1, although a soft-consensus variant recovers most of this cost (Appendix~\ref{appendix:llama_analysis}).

    \paragraph{3) Cultural Knowledge.}

        Finally, we test the hypothesis of \emph{cultural erasure} by evaluating model performance on the original BLEnD benchmark, which probes intrinsic cultural knowledge.
        To prevent data contamination, we restrict evaluation to BLEnD sample IDs present in our test set, subsampling a maximum of eight questions per ID and country to ensure computational tractability ($\sim$6,600 samples). \autoref{tab:c3po_implications} (iii) and (iv) present BLEnD accuracy on the eight countries included in our training data versus the remaining countries present only in BLEnD.
        Performance remains largely stable from Vanilla to C-3PO Vanilla, fluctuating between $(-2.4\%, +0.3\%)$. Llama is again the exception, losing 7.8 and 5.5 points on seen and unseen countries, most of which a soft-consensus variant recovers (Appendix~\ref{appendix:llama_analysis}).

    \paragraph{Takeaways.}

        Cultural bias, persona adherence and intrinsic cultural knowledge remain highly stable following C-3PO fine-tuning (excluding Llama, whose cost a soft-consensus variant largely recovers; Appendix~\ref{appendix:llama_analysis}). Crucially, \textbf{C-3PO achieves the highest cross-lingual consistency among the methods we compare whilst preserving explicit user identities and cultural neutrality.}

\section{Layer-wise Interpretability Analysis}
    \label{sec:interpretability_analysis}

    \figCombinedLayerSlopeAnalysis

    We hypothesise that Cross-lingual Cultural Inconsistency stems from implicit cultural personalisation, where the model tailors its response to the stereotypical culture associated with the input language (e.g., a Spanish prompt eliciting Mexican cultural answers).
    To understand the mechanisms driving this effect, we conduct a layer-wise interpretability analysis on all three models using our test set, presenting Llama-3.1-8B in detail and the other two models in \autoref{appendix:cross_model_interpretability}.

    We first decode layer-wise predictions by eliciting direct multiple-choice outputs instead of a structured format (\autoref{fig:blend_mcq_format_no_json}).
    Intermediate representations are routed directly through the final layer normalisation and language modelling head to obtain a probability distribution over the vocabulary, a projection also used to study the latent languages of multilingual models \cite{zhong_what_2025}.
    Finally, we map the model's intermediate answer predictions back to their corresponding countries using BLEnD's annotations, allowing us to track the evolution of cultural personalisation at every layer.

    \autoref{fig:combined_layer_slope_analysis} (left) visualises the frequency of predictions for the language-specific stereotypical country across layers.
    Stereotypical predictions surge in later layers, indicating an escalating cultural personalisation effect.
    This shift coincides precisely with the consistency jump around layers 22--25 (\autoref{fig:layer_wise_kappa}), suggesting that the model commits to a culturally personalised answer exactly as its representation stabilises.
    We also notice a higher consistency for the Higher-Resource group that emerges earlier in the network compared to other groups, and a pronounced bias towards the US (\autoref{fig:layer_wise_country_predictions}), corroborating prior findings on Western-centric biases in LLMs \cite{bulte_llms_2025, zhou_does_2025}.

    A potential confounding factor is that the model might simply increase predictions for a specific country across all input languages in later layers.
    To confirm that the personalisation effect is stronger for the stereotypical language, we perform a slope analysis by fitting a linear regression to the prediction frequency of each country across layers, conditioned on the input language.
    \autoref{fig:combined_layer_slope_analysis} (right) presents these slopes (\%) for each language-country pair. With a single exception (Algeria), the slope for any given country is strictly highest when queried in its stereotypical language.
    These findings support the hypothesis that \textbf{MLLMs implicitly personalise their answers towards the prompt language's stereotypical culture}, suggesting that this personalisation effect is a key driver of cross-lingual inconsistency.

    \paragraph{Other Models.}
        Qwen-2.5-3B and Gemma-2-27B show the same pattern (\autoref{fig:cross_model_slope_analysis}), with the stereotypical language among the top two for its country in six of eight cases; exceptions reflect English and US biases and co-occurring cultures (\autoref{appendix:cross_model_interpretability}).

\section{Conclusion}

    We systematically investigated Cross-lingual Cultural Inconsistency (CCI) in MLLMs,
    demonstrating its entanglement with language-driven personalisation and its greater severity in lower-resource languages.
    To robustly quantify this, we introduced \textit{Singleton Fleiss's $\kappa_S$}, a metric resilient to hallucinations. For mitigation, we proposed \textit{Cross-lingual Cultural Consistent Preference Optimisation} (C-3PO), a self-supervised, consensus-driven alignment framework that outperforms the average of every baseline
    while largely preserving cultural neutrality, explicit user identities and intrinsic cultural knowledge; where it does not, a soft-consensus variant recovers most of the cost. C-3PO's gains also transfer to open-ended generation. Finally, our layer-wise interpretability analysis revealed that, as forward-pass representations stabilise, MLLMs implicitly anchor outputs to the prompt language's stereotypical culture.

\newpage

\section*{Limitations}
    \addcontentsline{toc}{section}{Limitations}

    Our methodology relies on three open-weight models of up to 27B parameters, selected for their multilingual capabilities, their amenability to latent analysis and fine-tuning and computational feasibility. Consequently, our findings may not fully represent the behaviour of substantially larger, potentially proprietary, models.
    Furthermore, our evaluation is restricted to a dataset constructed from the BLEnD benchmark. While our automated assurance with manual verification indicates high quality, the treatment of questions and translation of answer options using GPT-5.2 may introduce subtle biases. Future work should focus on validating these rigorously through manual review, and explore generalisability across diverse datasets, including those beyond everyday knowledge.

    {
        To automate the data processing pipeline, we used GPT-5.2, the state-of-the-art model at the time, and GPT-5.4, released later, for quality assurance; each was the most capable model available when that stage was run.
        We acknowledge that this hampers the complete reproducibility of the data processing, as it relies on a specific closed-source model.
        To mitigate this issue and diminish the need for complete reconstruction, we release all prompts and data in the supplementary material.}

    {
        The results on cultural and persona alignment (\autoref{tab:c3po_implications}) show that C-3PO, while improving cross-lingual consistency, maintains stable cultural bias, persona adherence and cultural knowledge.
        However, as noted in the text, C-3PO lowers Llama's persona adherence and BLEnD performance.
        We trace this cost to Llama's stronger language-driven personalisation and narrower multilingual coverage, and a soft-consensus variant recovers most of it (Appendix~\ref{appendix:llama_analysis}); attributing this recovery to individual components is left to future work.
        Additionally, overall BLEnD performance fluctuates slightly from Vanilla to C-3PO, with some models showing a small drop and others a small improvement. This suggests that the trade-offs between consistency mitigation and cultural erasure may be model-specific. This warrants further investigation, alongside a more rigorous evaluation of what constitutes a significant or acceptable accuracy drop in exchange for alignment.}

    In our work, we restrict the scope of our analysis to eight languages from BLEnD's benchmark and their associated countries as proxies for culture.
    They span a diverse set of resource levels, scripts and language families, and their number was chosen for convenience, to fit exactly eight language variations of a query into a single C-3PO preference batch (a power of 2, optimal for GPU memory constraints); the method itself places no constraint on the number of languages beyond computational cost.
    On the one hand, this is an intrinsic limitation of our work, as it restricts the generalisability of our findings to other languages and cultures not included in our analysis.
    On the other hand, \autoref{tab:c3po_implications} (iv) shows that C-3PO's effect on countries outside its training data ranges from a small gain to a loss, depending on the model. We leave further investigation of this scaling to future work.

    Our main experiments use multiple-choice questions, whose categorical answers $\kappa_S$ requires. Our open-ended evaluation (Appendix~\ref{sec:open_ended}) maps free-form answers to categories with rule-based lexicons, so its $\kappa_S$ is a lower bound: two answers that agree in words no lexicon covers count as a disagreement. Its 79 base questions also limit statistical power, and C-3PO itself is still trained on multiple-choice data only. Likewise, our consensus-rule ablation keeps the hyperparameters selected for C-3PO: it shows that softer consensus does not raise consistency at these settings, not that it cannot at others.

    Theoretically, our framework equates countries with cultures -- a common yet reductive proxy \cite{veselovsky_localized_2025, pawar_presumed_2025}. Future studies must adopt more multidimensional representations of culture \cite{liu_culturally_2025, pawar_survey_2025}. Additionally, while our slope analysis (\autoref{fig:combined_layer_slope_analysis}) successfully isolates language-specific personalisation, it simultaneously reveals a competing, pervasive US-centric bias across several languages. The interplay between this systemic Western bias and language-driven personalisation warrants further investigation, which we leave to future work.

    Finally, our work studies cross-lingual consistency under the strict assumption of explicitly provided user personas. Outside this scope -- locale-indexed queries asked without a persona -- adapting to the prompt's language is often appropriate, and our findings do not argue for consistency there. We also advocate for deeper mechanistic interventions, such as causal activation patching \cite{yu_entangled_2026}, to definitively isolate the internal drivers of cross-lingual cultural inconsistency and better understand the dynamics of consensus-driven alignment; in particular, whether light-weight fine-tuning such as C-3PO alters the representations in which inconsistency arises during pretraining, or only the behaviour built on them, remains open.

    Consistency training could also be misused to suppress legitimate cultural variation under the guise of consistency, for instance by imposing one culture's answers across all languages. C-3PO forms its targets from the model's own answers without a pivot language, and we find no concentration of the fine-tuned models' answers on any country (\autoref{sec:cultural_and_persona_alignment}), but the same machinery could be given a chosen target instead of a consensus. We therefore recommend applying consistency training only where an explicit persona defines the expected culture, and auditing country distributions before deployment.

\section*{Acknowledgements}

    We thank Ivan Vulić for his early guidance. We are grateful to Tiancheng Hu, Han Zhou and Yinhong Liu for their brainstorming and feedback, which helped shape the early conceptualisation of this work. We also thank Cecilia Liu for her feedback during the project's early stages, alongside our lab colleagues for informal discussions.

    This work was supported by the UK Research and Innovation (UKRI) Frontier Research Grant EP/Y031350/1 EQUATE.

    \noindent $\begin{array}{l}\includegraphics[width=1cm]{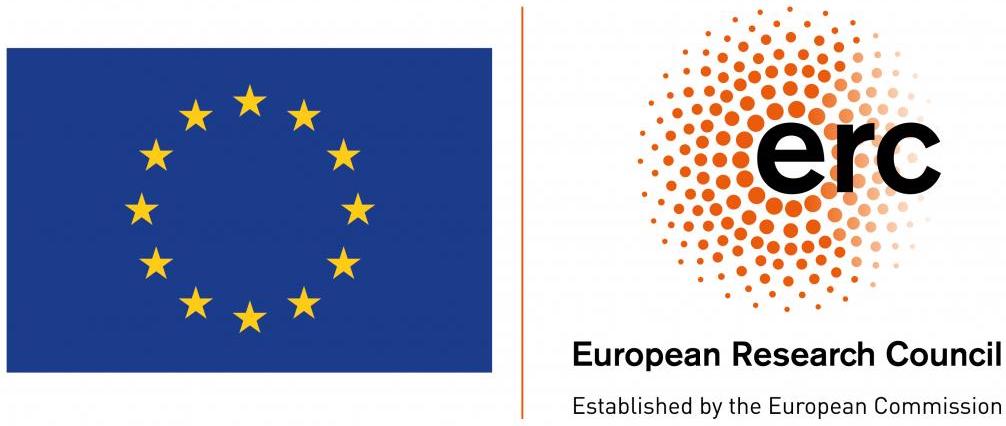} \end{array}$ 
    This research was co-funded by the European Union (ERC, ExplainYourself, 101077481), and supported by the Pioneer Centre for AI, DNRF grant number P1. Views and opinions expressed are however those of the author(s) only and do not necessarily reflect those of the European Union or the European Research Council. Neither the European Union nor the granting authority can be held responsible for them.

    Lucas Resck gratefully acknowledges funding from the Cambridge Commonwealth, European and International Trust through a PhD scholarship. Lucas Resck acknowledges travel support from ELIAS (GA no 101120237). Furthermore, Lucas Resck acknowledges funding from the Danish Data Science Academy (DDSA).

    AI tools were employed to assist with specific tasks, including coding, text refinement and information summarisation, enhancing overall workflow efficiency. The authors meticulously reviewed all AI-assisted outputs and bear full responsibility for the final content of this manuscript.
    
    Emoji graphics used in figures are provided by \href{https://github.com/googlefonts/noto-emoji}{Google Noto Emoji}, licensed under the \href{https://www.apache.org/licenses/LICENSE-2.0}{Apache License 2.0}.



    \appendix
    \renewcommand{\subsectionautorefname}{Appendix}
    \renewcommand{\subsubsectionautorefname}{Appendix}

    \section{Consistency Metric Definitions}
    \label{appendix:metric_definitions}

    \begin{definition}[Original Fleiss's $\kappa$ for Cross-lingual Consistency]
        \label{def:kappa_for_cross_lingual_consistency}
        Let $N$ be the total number of samples, $n$ be the number of evaluation languages, and $k$ be the number of valid answer categories per sample. We formulate the observed agreement ($P_o$) and expected chance agreement ($P_e$) as follows:
        \begin{align*}
            P_o &= \frac{1}{Nn(n-1)} \sum_{i=1}^N \sum_{j=1}^k n_{ij}(n_{ij} - 1) \\
            P_e &= \sum_{j=1}^k \left( \frac{1}{Nn} \sum_{i=1}^N n_{ij} \right)^2
        \end{align*}
        where $n_{ij}$ denotes the number of languages that assigned the $j$-th category to the $i$-th sample. The consistency metric $\kappa$ is defined as:
        \begin{equation*}
            \kappa = \frac{P_o - P_e}{1 - P_e}.
        \end{equation*}
    \end{definition}

    \begin{definition}[Soft Consistency, Hard Consistency and Mode Frequency]
        Following the notation in \autoref{def:kappa_for_cross_lingual_consistency}, letting $a_{il}$ denote the answer generated for the $i$-th sample in the $l$-th language, and $n_{ij}$ be the count of languages that selected the $j$-th category for the $i$-th sample, we define:
        \begin{align*}
            \text{Soft} &= \frac{2}{Nn(n-1)} \sum_{i=1}^N \sum_{1 \leq l < m \leq n} \mathbb{I}[a_{il} = a_{im}], \\
            \text{Hard} &= \frac{1}{N} \sum_{i=1}^N \mathbb{I}\left[\bigwedge_{l=1}^{n-1} a_{il} = a_{i(l+1)}\right], \\
            \text{Mode} &= \frac{1}{N} \sum_{i=1}^N \frac{\max_{j} n_{ij}}{n}.
        \end{align*}
    \end{definition}

\section{Proof of Convergence of $\kappa_S$}
    \label{appendix:proof_singleton_kappa}

    In this section, we demonstrate that Singleton Fleiss's $\kappa$ ($\kappa_S$) (\autoref{def:singleton_fleiss_kappa}) converges to the standard $\kappa$ (determined exclusively by valid categories; see \autoref{def:kappa_for_cross_lingual_consistency}) as the sample size $N$ approaches infinity. Here, $\kappa$ applies \autoref{def:kappa_for_cross_lingual_consistency} to the valid categories while keeping all $Nn$ assignments in its denominators, so that each invalid response counts as a disagreement; without invalid responses, it is exactly Fleiss's $\kappa$.

    \begin{theorem}[Convergence of $\kappa_S$]
        \label{thm:convergence}
        Let $\kappa_S$ and $\kappa$ be defined as in Definitions \ref{def:singleton_fleiss_kappa} and \ref{def:kappa_for_cross_lingual_consistency}. As the dataset size grows, $\kappa_S$ asymptotically converges to the valid-category $\kappa$:
        \begin{equation*}
            \lim_{N \to \infty} \kappa_S = \kappa.
        \end{equation*}
    \end{theorem}

    \begin{proof}
        Let $N$ denote the number of samples and $n$ the number of raters (languages), yielding a total of $Nn$ assignments. Let $\mathcal{V}$ be the fixed set of valid categories and $\mathcal{U}$ be the set of dynamically generated singleton (invalid) categories. The complete categorical space is $\mathcal{V}' = \mathcal{V} \cup \mathcal{U}$.

        The formulation for Singleton Fleiss's $\kappa$ is:
        \begin{equation*}
            \kappa_S = \frac{P_o^S - P_e^S}{1 - P_e^S}
        \end{equation*}
        We analyse the asymptotic behaviour of the observed agreement $P_o^S$ and the expected agreement $P_e^S$ independently.

        \paragraph{1. Analysis of Observed Agreement ($P_o^S$).}
            The observed agreement represents the proportion of identical assignment pairs out of all possible pairs:
            \begin{equation*}
                P_o^S = \frac{1}{Nn(n-1)} \sum_{i=1}^{N} \sum_{j \in \mathcal{V}'} n_{ij}(n_{ij}-1)
            \end{equation*}
            where $n_{ij}$ denotes the number of raters assigning category $j$ to sample $i$. We partition the summation over $\mathcal{V}'$ into valid ($\mathcal{V}$) and singleton ($\mathcal{U}$) sets:
            \begin{align*}
                \sum_{j \in \mathcal{V}'} n_{ij}(n_{ij}-1) =& \sum_{v \in \mathcal{V}} n_{iv}(n_{iv}-1) \\
                &+ \sum_{u \in \mathcal{U}} n_{iu}(n_{iu}-1)
            \end{align*}
            By the operational definition of a singleton, any invalid response $u \in \mathcal{U}$ is unique to a specific sample and a specific rater. Therefore, for all $u \in \mathcal{U}$:
            \begin{equation*}
                n_{iu} = \begin{cases} 1, & \text{if } u \text{ is the generated response,} \\ 0, & \text{otherwise.} \end{cases}
            \end{equation*}
            Consequently, the combinatorial term $n_{iu}(n_{iu}-1)$ evaluates to $1(0) = 0$ in all cases. The summation over $\mathcal{U}$ vanishes identically:
            \begin{equation*}
                P_o^S = \frac{1}{Nn(n-1)} \sum_{i=1}^{N} \sum_{v \in \mathcal{V}} n_{iv}(n_{iv}-1)
            \end{equation*}
            This establishes that $P_o^S$ is strictly identical to the observed agreement calculated over valid categories, denoted $P_o^\mathcal{V}$. Thus, $P_o^S = P_o^\mathcal{V}$ for any finite $N$.

        \paragraph{2. Analysis of Expected Agreement ($P_e^S$).}
            Expected agreement is defined as the sum of squared marginal proportions across all categories:
            \begin{equation*}
                P_e^S = \sum_{j \in \mathcal{V}'} p_j^2 = \sum_{v \in \mathcal{V}} p_v^2 + \sum_{u \in \mathcal{U}} p_u^2
            \end{equation*}
            where $p_j$ is the proportion of the $Nn$ assignments categorised as $j$. Because each singleton $u \in \mathcal{U}$ appears exactly once within the entire dataset, its marginal proportion is inherently:
            \begin{equation*}
                p_u = \frac{1}{Nn}
            \end{equation*}
            Let $M_{\mathcal{U}} = |\mathcal{U}|$ represent the total absolute frequency of invalid responses across the dataset. The sum of squared proportions for the singleton set becomes:
            \begin{equation*}
                \sum_{u \in \mathcal{U}} p_u^2 = \sum_{u \in \mathcal{U}} \left( \frac{1}{Nn} \right)^2 = \frac{M_{\mathcal{U}}}{(Nn)^2}
            \end{equation*}
            We parameterise the model's global error rate as $\epsilon = M_{\mathcal{U}} / (Nn)$. Substituting $M_{\mathcal{U}} = \epsilon Nn$ into the equation yields:
            \begin{equation*}
                \sum_{u \in \mathcal{U}} p_u^2 = \frac{\epsilon Nn}{(Nn)^2} = \frac{\epsilon}{Nn}
            \end{equation*}
            Assuming the error rate $\epsilon$ is bounded ($\epsilon \in [0, 1]$), taking the limit as the sample size $N \to \infty$ yields:
            \begin{equation*}
                \lim_{N \to \infty} \sum_{u \in \mathcal{U}} p_u^2 = \lim_{N \to \infty} \frac{\epsilon}{Nn} = 0
            \end{equation*}
            Therefore, the expected agreement asymptotically converges to the sum of squared proportions of the valid categories alone:
            \begin{equation*}
                \lim_{N \to \infty} P_e^S = \sum_{v \in \mathcal{V}} p_v^2 = P_e^\mathcal{V}
            \end{equation*}

        \paragraph{Conclusion.}
            Write $s = \epsilon / (Nn)$ for the singleton term of step 2, so that $P_o^S = P_o^\mathcal{V}$ and $P_e^S = P_e^\mathcal{V} + s$. Substituting these into $\kappa$ and $\kappa_S$ gives their exact difference and, as $0 \leq 1 - P_o^\mathcal{V} \leq 1$ and $P_e^\mathcal{V} \leq P_e^S < 1$, bounds it:
            \begin{equation*}
                0 \leq \kappa - \kappa_S = \frac{s\,(1 - P_o^\mathcal{V})}{(1 - P_e^\mathcal{V})(1 - P_e^S)} \leq \frac{\epsilon}{Nn\,(1 - P_e^S)^2},
            \end{equation*}
            which vanishes as $N \to \infty$ while $P_e^S$ stays bounded away from one.
            By synthesising the limits for the observed and expected agreement terms, we conclude:
            \begin{equation*}
                \lim_{N \to \infty} \kappa_S = \lim_{N \to \infty} \frac{P_o^S - P_e^S}{1 - P_e^S} = \frac{P_o^\mathcal{V} - P_e^\mathcal{V}}{1 - P_e^\mathcal{V}} = \kappa
            \end{equation*}
            This concludes the proof.
    \end{proof}

\section{Dataset Construction Details}
    \label{appendix:dataset_construction_details}

    To systematically evaluate cross-lingual cultural inconsistency, we construct a parallel dataset of culturally relevant questions across multiple languages. To achieve this, we transform the BLEnD benchmark~\cite{myung_blend_2024} into a multilingual multiple-choice question (MCQ) dataset (see \autoref{fig:dataset_construction} for an overview of the process).

    BLEnD provides a comprehensive collection of everyday knowledge queries spanning diverse cultures and languages. We specifically isolate its MCQ subset, as the discrete answer space facilitates the application of standard statistical agreement metrics, such as Fleiss's $\kappa$, to quantify consistency.
    BLEnD MCQ prompts are in the following format: a question (e.g., ``What is the most popular sport in Mexico?'') followed by a set of answer options, each associated with a specific country (e.g., ``football'' for Mexico, ``baseball'' for United States, etc. -- see \autoref{fig:blend_mcq_format} for an example). This country-level annotation is critical for our subsequent interpretability analysis of implicit model personalisation (\autoref{sec:interpretability_analysis}).
    To balance linguistic and cultural diversity -- spanning varied language resource levels, scripts and geographical distributions -- while maintaining computational tractability, we process the data for eight languages (English, Spanish, Chinese, Arabic, Indonesian, Korean, Greek and Persian) and their natively associated countries (United States, Mexico, China, Algeria, Indonesia, South Korea, Greece and Iran, respectively). We then filter the dataset to retain only queries featuring at least one valid answer mapping to these eight countries, leveraging BLEnD's country-level annotations.

    Because the original BLEnD MCQ subset is exclusively in English, we reconstruct the multilingual queries by leveraging the parallel ground-truth annotations (questions only) from BLEnD's short-answer questions (SAQ), which are available in all eight target languages.
    Essentially, BLEnD MCQ and SAQ subsets share the same base questions but SAQ provides the question text in all languages.
    A critical methodological step involves neutralising the question prompts. Original BLEnD questions contain explicit geographical markers (e.g., ``What is the most popular sport in Mexico?''), which would act as confounding variables when attempting to isolate \textit{language-driven} personalisation. To resolve this, we employ GPT-5.2~\cite{openai_gpt-52_2025} to systematically strip all explicit country or regional references from the text, ensuring the prompts are culturally neutral while remaining grammatically coherent (see \autoref{fig:treat_question_prompt} for the prompt template).
    For example, the query ``What is the most popular sport in Mexico?'' is transformed into ``What is the most popular sport?''.
    This renders the question culture-agnostic and removes explicit geographical cues, allowing us to strictly isolate implicit personalisation effects driven solely by the prompt's language.

    In BLEnD, multiple MCQ samples often share the same base question but feature varying combinations of answer options. We designate the shared base question as a ``supersample'' and its specific option variants as ``subsamples''.
    For each supersample, we iteratively draw subsamples (same question with different answer options) until the options exhaustively include the eight selected countries.
    Because short-answer questions in BLEnD only include translated questions (no translated answer options), the corresponding MCQ answer options, originally in English, are subsequently translated into all eight languages using GPT-5.2 (see \autoref{fig:translate_options_prompt}).
    Both the question neutralisation and option translation processes are manually verified in a small subset of samples to ensure quality and consistency (\autoref{appendix:dataset_quality_assurance}).
    Finally, the compiled dataset, with 15,840 prompts -- 1,980 multiple-choice items per language, derived from 393 base questions --, is partitioned into training, validation and test sets using a 70\%-10\%-20\% split of the base questions (275, 39 and 79 questions, i.e. 1,394, 195 and 391 items). To strictly preclude data leakage, we enforce a supersample-level split, ensuring that all subsamples of a supersample are assigned to the same dataset partition.
    Furthermore, all cross-lingual translations of a given sample are strictly assigned to the identical partition.

    \paragraph{Model Versions and Decoding.} All dataset-construction calls use GPT-5.2~\cite{openai_gpt-52_2025} through the OpenAI Responses API, with reasoning effort set to \texttt{none}, verbosity set to \texttt{low} and all other sampling parameters at their API defaults. Quality assurance (\autoref{appendix:dataset_quality_assurance}) uses GPT-5.4~\cite{openai_gpt-54_2026} with reasoning effort set to \texttt{high} and structured outputs. GPT-5.4 was released after the dataset had been built; each model was the most capable one available when its stage was run. As API models may change, we release the constructed dataset itself, so that reproducing our results does not require re-running this pipeline.

\section{Dataset Quality Assurance}
    \label{appendix:dataset_quality_assurance}

    To ensure the quality of the dataset, we implemented an Automated Quality Assurance Pipeline with Author Verification. We employed GPT-5.4~\cite{openai_gpt-54_2026} (high reasoning) to assess the country neutralisation and the answer option translation. The model was instructed to verify the quality of the result (acceptable vs. non-acceptable) and provide a score when it was non-acceptable.

    We analysed a subset of 240 samples (30 $\times$ 8 languages). We manually reviewed all negative results alongside a 5\% random sample of the positive ones. We identified failed country neutralisation in 1.67\% of samples (95\% CI: $[0.65\%, 4.21\%]$) and minor option translation discrepancies in 7.08\% (95\% CI: $[4.47\%, 11.05\%]$). Even among these automatically flagged discrepancies, the LLM-assigned translation quality scores remained high (averaging 0.8/1.0), indicating that the semantic drift in the dataset is minimal and non-destructive.

\section{Prompts}
    \label{appendix:prompts}

    \figBLEnDMCQFormat

    \figTreatQuestionPrompt

    \figTranslateOptionsPrompt

    \figPersonaPrompt

    \figBLEnDMCQFormatNoJSON

    We outline the prompt templates employed throughout our methodology, including the original BLEnD MCQ format (\autoref{fig:blend_mcq_format}), whose default JSON output format we use to prompt all models, the question neutralisation prompt (\autoref{fig:treat_question_prompt}), the option translation template (\autoref{fig:translate_options_prompt}), the persona prompt (\autoref{fig:persona_prompt}) and the BLEnD MCQ format without JSON parsing (\autoref{fig:blend_mcq_format_no_json}).

\section{Hyperparameters}
    \label{appendix:hyperparameters}

    \tabHyperparameters

    Comprehensive hyperparameter configurations for all mitigation strategies and baselines evaluated in this work are detailed in \autoref{tab:hyperparameters}.
    For C-3PO, each of the nine configurations per model (three ranks and three epoch counts) is trained once with the trainer's default seed (42), and the configuration with the highest validation $\kappa_S$ is evaluated once on the test set; the consensus-rule variants of \autoref{sec:consensus_ablation} reuse it. Rejected options for already-consistent languages and the undersampling draws use Python's \texttt{random} module without a fixed seed, so we release the resulting preference pairs. \autoref{tab:training_data} reports the selected configurations and the retained training data: every retained query contributes one pair per language, and the pairs from divergent languages show how evenly undersampling spreads the corrections across languages.

    \tabTrainingData

\section{Interpretability Across Models}
    \label{appendix:cross_model_interpretability}

    \figCrossModelSlopeAnalysis

    For Qwen-2.5-3B and Gemma-2-27B, the analysis of \autoref{sec:interpretability_analysis} supports the same conclusion (\autoref{fig:cross_model_slope_analysis}). Stereotypical predictions again surge in late layers (30--35 of Gemma's 46 and 30--31 of Qwen's 36) and coincide with the consistency jumps; the Higher-Resource group again becomes consistent earlier, most clearly for Qwen; and the bias towards the US is stronger still, especially for Gemma.
    As the models differ in depth, we compare slopes over the full depth rather than per layer. The stereotypical language ranks first or second for its country in seven of eight cases for Llama and six of eight for Qwen and Gemma, and the stereotypical slope exceeds the same country's slope in other languages most for Llama (by 4.4 percentage points over the full depth, against 3.2 for Gemma and 1.9 for Qwen).
    The exceptions differ between models and follow three patterns: an \textit{English bias}, whereby lower-resource cultures are more readily expressed in English than in their own language (Algeria in English, for Llama and Gemma); a \textit{US bias}, whereby the US attains the steepest slope for other languages (Arabic, in all three models); and \textit{cultural or geographic co-occurrence}, whereby related cultures confound one another (Algeria in Greek, for Llama and Qwen; China in Indonesian, for Gemma). Llama's personalisation, the strongest of the three, is revisited in \autoref{appendix:llama_analysis}.

    \section{Open-Ended Generation and Consensus Ablation}
    \label{sec:further_analyses}

    We now take C-3PO beyond the setting of the main experiments: we test whether its gains carry over to open-ended generation (\autoref{sec:open_ended}), ablate its consensus rule (\autoref{sec:consensus_ablation}), and examine why it costs Llama-3.1-8B some persona adherence and cultural knowledge, and how that cost can be recovered (\autoref{appendix:llama_analysis}).

    \subsection{Open-Ended Generation}
        \label{sec:open_ended}

        \tabOpenEnded

        Multiple-choice questions make consistency measurable, but inconsistency may surface differently when models answer freely. We therefore ask whether C-3PO, trained only on multiple-choice questions, also improves consistency in open-ended generation.

        \paragraph{Setup.} We reuse the parallel questions of our test set without their answer options, instructing the model in each language to answer with a short phrase in the same JSON format. As the options are the only difference between a question's multiple-choice variants, the unit becomes the 79 base questions, each asked in eight languages. We evaluate the models, adapters and personas of the main experiments without retraining.

        \paragraph{Scoring.} To compute $\kappa_S$, free-form answers must be mapped to categories. We do so by rule-based matching against lexicons of increasing coverage: the question's own options in the response language (L0) or in all eight languages (L1); these plus BLEnD's answer variants (L2); L1 plus the options of all other questions, a cross-question lexicon of 2,364 concepts (L3); and L3 plus identity between otherwise unmatched answers (L4). As in the multiple-choice setting, an answer that matches nothing becomes a singleton. We report L4 as the primary level and all levels in \autoref{appendix:open_ended}. L2 is excluded from the primary analysis because BLEnD records variants only in each country's own language and in English, which would favour the stereotypical language--country pairs under study.

        \paragraph{Results.} C-3PO improves open-ended consistency in 29 of 30 comparisons across models, language groups and the two method pairs (\autoref{tab:open_ended}), 12 of them significantly. The gain is positive at every scoring level except the strictest (L0) for Gemma, and it is not an artefact of answers becoming more option-like: C-3PO makes Llama's and Gemma's answers less, not more, matchable by the lexicon while making them more consistent (\autoref{appendix:open_ended}).

        \paragraph{Comparison with the Multiple-Choice Setting.} Open-ended $\kappa_S$ is far lower than in \autoref{tab:main_results}. The metric is the same, but chance agreement nearly vanishes when answers are unconstrained ($P_e \leq 0.006$, against $0.25$ with four options), so $\kappa_S$ approaches raw agreement, and producing the same answer is harder than selecting it. The comparable quantity is the relative gain over Vanilla, which is of the same order in both settings, between 13\% and 31\% (\autoref{tab:open_ended_comparison}). \textbf{Training on discrete answers therefore transfers to free-form generation}, which suggests that C-3PO changes the models' behaviour rather than only their choice among options.

    \subsection{Ablating the Consensus Rule}
        \label{sec:consensus_ablation}

        \tabConsensusAblation

        C-3PO forms its target from a plurality vote over the models' greedy answers (Section~\ref{sec:c_3po}). We ablate this choice against a stricter vote and against a \textit{soft} consensus built from the models' own answer distributions.

        \paragraph{Variants.} \textit{Strict majority} requires at least five of eight languages to agree. The \textit{soft vote} sums, across languages, the probability each assigns to the four options, read at the step at which the model emits its answer letter; the consensus is the option with the highest total. On top of it, we add \textit{confidence weighting}, which weights each language by one minus the normalised entropy of its distribution; \textit{hard negatives}, which take a language's own runner-up option as the rejected response instead of a random one; and a \textit{margin gate}, which discards queries whose two leading options differ by less than 0.15 in the aggregated distribution (10--20\% of training queries, depending on the model). Each variant is trained with the hyperparameters selected for C-3PO on that model (\autoref{appendix:hyperparameters}), so only the preference data differs.

        \paragraph{Results.} No variant improves consistency on any model, with multiple-choice or open-ended questions: none of the 30 changes in \autoref{tab:consensus_ablation} is significantly positive, and the only significant changes are two losses on Gemma. Strict majority improves on plurality by at most 0.004, while finding no majority for 14\% to 33\% of training queries (against 5\% to 10\% under plurality, \autoref{appendix:consensus_ablation}). Conversely, the soft vote always reaches a consensus and builds 10--17\% more preference pairs, yet gains nothing, which suggests that the queries a plurality vote discards carry little useful signal. At C-3PO's hyperparameters, the plurality rule is therefore at least as good as the alternatives we tested. The soft variants matter elsewhere, however, as \autoref{appendix:llama_analysis} shows.

    \subsection{Llama's Cost, and How to Recover It}
        \label{appendix:llama_analysis}

        \tabLlamaAnalysis

        C-3PO leaves Qwen's and Gemma's persona adherence and cultural knowledge largely intact, but not Llama's (\autoref{tab:c3po_implications}): it lowers Llama's persona adherence by 6.8 points and its BLEnD accuracy by 7.8 and 5.5 points on countries seen and unseen in training. Neither knowledge loss is noise: resampling base questions, both 95\% intervals exclude zero.

        \paragraph{Why Llama?} Our interpretability analysis offers a hypothesis. Among the three models, Llama shows the strongest language-driven cultural personalisation, with the largest excess of each stereotypical country's slope over its slope in other languages (\autoref{appendix:cross_model_interpretability}). Pulling all languages towards a single consensus must then override more of this personalisation, perturbing the culture-related representations on which persona adherence and cultural knowledge depend. This is compounded by Llama's narrower multilingual coverage: it officially supports eight languages, only two of which are among ours \cite{grattafiori_llama_2024}, whereas Qwen supports over 29 \cite{qwen_qwen25_2025}, and it is considerably smaller than Gemma.

        \paragraph{Recovering the Cost.} If the cost comes from pushing Llama towards consensus targets it holds with little confidence, a softer consensus should reduce it. The full soft-consensus variant of \autoref{sec:consensus_ablation} -- confidence-weighted vote, hard negatives and margin gate -- does so: it recovers about three quarters of Llama's persona-adherence loss and two thirds to three quarters of its knowledge loss (\autoref{tab:llama_analysis}), without a significant change in consistency (\autoref{tab:consensus_ablation}). Llama's remaining knowledge cost is then comparable to Gemma's, and the variant leaves Gemma unchanged while slightly improving Qwen's knowledge of seen countries. Nor is the recovery a matter of training on less data: the variant uses only 6\% fewer preference pairs than C-3PO for Llama, too few to explain it. As the three soft components are not separated here, we attribute the recovery to the variant as a whole; it makes the soft consensus a practical option whenever C-3PO costs a model adherence or knowledge.

\section{Open-Ended Evaluation Details}
    \label{appendix:open_ended}

    \paragraph{Prompt.} The open-ended prompt keeps the question and the answer cue of the multiple-choice prompt (\autoref{fig:blend_mcq_format}), removes the options, and replaces the instruction to choose one of them with \textit{Without any explanation, answer with a short phrase only. Provide as JSON format: \{``answer'':``''\}}, written in each of the eight languages. Persona methods add the persona prompt of \autoref{fig:persona_prompt}. Answers are decoded greedily, with at most 64 new tokens.

    \paragraph{Scoring.} Matching is rule-based, with no embeddings or LLM judges. Answers are normalised (Unicode normalisation, case folding, and removal of diacritics and punctuation), list-like answers are scored by their first item, and each answer takes the candidate with the longest match, subject to word-boundary and coverage checks. An answer that matches two candidates equally well is left unmatched and, like every unmatched answer, becomes a singleton. The share of answers that can agree with another rises from 20--29\% at L0 to 92--94\% at L4. Identity matching at L4 accounts for a minority of the agreement: for Qwen, Llama and Gemma, 59\%, 75\% and 79\% of C-3PO's increase in agreeing language pairs is lexicon-based, and the answers shared by identity alone are specific, such as the names of athletes, rather than filler. The check behind excluding L2 is mixed but supports it: for Llama and Gemma, adding BLEnD's variants matches two to three times more answers to the response language's own country than to other countries, although these gains are all below 1.3 points and the pattern is reversed for Qwen.

        \tabOpenEndedLevels

    \paragraph{Scoring Levels.} \autoref{tab:open_ended_levels} reports the gain at every level. It grows from L0 to L4 because unmatched answers count as disagreement for both methods, which compresses the difference between them; each level removes part of this floor.

        \tabOpenEndedOptionShape

    \paragraph{Option Shape.} If C-3PO improved open-ended consistency only by making answers more option-like, more of its answers would match the lexicon. \autoref{tab:open_ended_option_shape} shows the opposite for Llama and Gemma, whose lexicon matches fall by 4--5 points, while Qwen's are unchanged.

        \tabOpenEndedComparison

    \paragraph{Comparison with Multiple Choice.} \autoref{tab:open_ended_comparison} decomposes $\kappa_S$ in both settings. Chance agreement falls from about 0.25 to below 0.01 without options, so the lower open-ended $\kappa_S$ reflects lower observed agreement, not a change of scale; the relative gain is the comparable quantity.

\section{Consensus Ablation Details}
    \label{appendix:consensus_ablation}

    All variants read each language's answer distribution while the model answers the training queries greedily: at the first generated token that is an option letter, the probabilities of the four letters are renormalised to sum to one. The answers counted by the hard rules are the most probable letters of these distributions, so all rules draw on the same model outputs; a query is skipped if, in some language, the model never emits an option letter. A language's confidence weight is one minus the entropy of its distribution divided by $\log 4$, the aggregated distribution is the weighted average of the languages' distributions, and the margin gate compares its two leading options. In every variant, a language whose answer differs from the consensus is rejected with that answer, and the language-balanced undersampling of Section~\ref{sec:c_3po} is applied with a fixed seed. \autoref{tab:consensus_pairs} reports the training data of every variant.

    \tabConsensusPairs

\section{Additional Experimental Results}
    \label{appendix:additional_experimental_results}

    We provide supplementary experimental results that complement the findings presented in the main text. These encompass consistency variation plots across all models (\autoref{fig:consistency_variation_plots}) alongside expanded breakdowns of the primary results (\autoref{tab:main_results}) for Soft Consistency (\autoref{tab:appendix_soft_consistency}), Hard Consistency (\autoref{tab:appendix_hard_consistency}), Mode Frequency (\autoref{tab:appendix_mode_frequency}) and overall Error Metric (\autoref{tab:appendix_error_metric}).
    We also present layer-wise analyses of consistency (\autoref{fig:layer_wise_kappa}), a more complete analysis of layer-wise country predictions (\autoref{fig:layer_wise_country_predictions}) and an expanded breakdown of cultural bias across all countries (\autoref{tab:cultural_bias_expanded}).

    \figConsistencyVariationPlots

    \tabCulturalBiasExpanded

    \figLayerWiseKappa

    \figLayerWiseCountryPredictions

    \clearpage

    \tabAppendixSoftConsistency

    \tabAppendixHardConsistency

    \tabAppendixModeFrequency

    \tabAppendixErrorMetric

\end{document}